\documentclass[final]{colt2020}

\usepackage[english]{babel}

\usepackage{wrapfig}
\usepackage{amsfonts}

\usepackage{xcolor}
\usepackage{colortbl}
\newcommand\mathTable{\fontsize{8}{9.6}\selectfont}

\usepackage{algorithm}
\usepackage{algorithmic}
\usepackage[utf8]{inputenc} 
\usepackage[T1]{fontenc}    
\usepackage{hyperref}       
\usepackage{url}            
\usepackage{booktabs}       
\usepackage{amsfonts}       
\usepackage{nicefrac}       
\usepackage{microtype}      
\usepackage{mathtools}
\usepackage{thmtools}
\usepackage{thm-restate}
\usepackage[title]{appendix}
\usepackage{comment}
\usepackage{dsfont}

\usepackage{times}

\usepackage{multirow}
\usepackage{makecell}
\usepackage{threeparttable}

\usepackage{array}
\newcommand{\PreserveBackslash}[1]{\let\temp=\\#1\let\\=\temp}
\newcolumntype{C}[1]{>{\PreserveBackslash\centering}p{#1}}
\usepackage{hhline}

\usepackage[font=small]{caption}

\newtheorem{myTheorem}{Theorem}

\newtheorem{theorem-rst}[myTheorem]{Theorem}
\newtheorem{lemma-rst}[theorem]{Lemma}
\newtheorem{proposition-rst}[myTheorem]{Proposition}
\newtheorem{assumption-rst}[theorem]{Assumption}
\newtheorem{claim-rst}[theorem]{Claim}
\newtheorem{corollary-rst}[theorem]{Corollary}

\DeclarePairedDelimiter\br{(}{)}
\DeclarePairedDelimiter\brs{[}{]}
\DeclarePairedDelimiter\brc{\{}{\}}

\DeclarePairedDelimiter\abs{\lvert}{\rvert}
\DeclarePairedDelimiter\norm{\lVert}{\rVert}

\DeclarePairedDelimiter\floor{\lfloor}{\rfloor}

\newcommand{\E}{\mathbb{E}}

\newcommand{\R}{\mathbb{R}}
\newcommand{\Scal}{\mathcal{S}} 
\newcommand{\Ocal}{\mathcal{O}} 

\newcommand{\action}{S} 
\newcommand{\actionset}{\Scal} 
\newcommand{\meanvec}{ {\boldsymbol{\mu}} } 
\newcommand{\meanval}{\mu} 
\newcommand{\rewardNoArgs}{r} 
\newcommand{\reward}[2]{\rewardNoArgs\br*{#1;#2}} 
\newcommand{\rewardvec}[1]{\rewardNoArgs\br*{#1}} 
\newcommand{\rdef}{\reward{\action}{\meanvec}} 

\newcommand{\xb}{x} 
\newcommand{\ub}{u} 
\newcommand{\ab}{a} 

\newcommand{\obs}[2][]{X^{(#2)}_{#1}}

\newcommand{\Nitems}{M}
\newcommand{\Nbatch}{K}
\newcommand{\Narms}{m}
\newcommand{\Ndis}{{\Nbatch_\commset}}
\newcommand{\commset}{I} 

\newcommand{\dr}[1]{\Delta_{#1}}

\newcommand{\onevec}{\mathbf{1}}

\newcommand{\kl}{D_\mathrm{KL}}
\newcommand{\klBin}{\mathrm{kl}}

\newcommand{\unu}{{\underline{\nu}}} 

\DeclareMathOperator*{\argmax}{\arg\max}

\newenvironment{proofsketch}
{
\par\noindent{\bfseries\upshape Proof Sketch\ }
}
{\jmlrQED}

\title[Tight Lower Bounds for Combinatorial Multi-Armed Bandits]{Tight Lower Bounds for Combinatorial Multi-Armed Bandits}

\coltauthor{%
 \Name{Nadav Merlis} \Email{merlis@campus.technion.ac.il}\\
 \addr Faculty of Electrical Engineering, Technion, Israel Institute of Technology.
 \AND
 \Name{Shie Mannor} \Email{shie@ee.technion.ac.il}\\
 \addr Faculty of Electrical Engineering, Technion, Israel Institute of Technology.\\
 \addr Nvidia Research.
}

\begin{document}

\maketitle

\begin{abstract}
The Combinatorial Multi-Armed Bandit problem is a sequential decision-making problem in which an agent selects a set of arms at each round, observes feedback for each of these arms and aims to maximize a known reward function of the arms it chose. While previous work proved regret upper bounds in this setting for general reward functions, only a few works provided matching lower bounds, all for specific reward functions. In this work, we prove regret lower bounds for combinatorial bandits that hold under mild assumptions for all smooth reward functions. We derive both problem-dependent and problem-independent bounds and show that the recently proposed Gini-weighted smoothness parameter \citep{merlisM19} also determines the lower bounds for monotone reward functions. Notably, this implies that our lower bounds are tight up to log-factors.
\end{abstract}

\begin{keywords}%
  Combinatorial Multi-Armed Bandits, Lower Bounds, Gini-Weighted Smoothness
\end{keywords}

\section{Introduction}

Combinatorial Multi-Armed Bandits (CMABs) are a well-known extension of Multi-Armed Bandits (MABs) \citep{robbins1952some}, where instead of choosing a single arm at each round, the agent selects a set of arms. It then observes noisy feedback for each arm in this set (`semi-bandit feedback') and aims to maximize a known reward function of the selected arms and their parameters. More specifically, it aims to minimize its regret, which is the expected cumulative difference between the reward of the best action and the reward of the agent's actions. The applications of this framework are numerous and vary between reward functions; the most common one is the linear reward function \citep{kveton2015tight}, which can be applied for problems such as spectrum allocation, shortest paths, routing problems and more \citep{gai2012combinatorial}. Another common application is the Probabilistic Maximum Coverage (PMC) problem \citep{merlisM19}, which is closely related to problems such as influence maximization and ranked recommendations.

Due to its usefulness, many previous works analyze regret upper bounds for different variants of this setting. While some works focus on specific reward functions, others derive bounds that hold for general reward functions. In these cases, the bounds usually depend on some measure of smoothness of the reward, for example, its global Lipschitz constant, or its Gini-weighted smoothness. The latter is a more refined smoothness criterion, recently suggested in \citep{merlisM19}, that takes into account the interaction between the local gradients of the reward and concentration properties of the arms. On the other hand, there are almost no works on matching lower bounds; to the best of our knowledge, all existing lower bounds for CMABs were derived for specific reward functions -- either the linear one or the PMC problem. Notably, there is no characterization of lower bounds for general reward functions, and it is unclear whether existing upper bounds are tight.

The gain from general lower bounds is threefold: (i) When the bounds are loose, understanding which quantities affect the lower bounds allows devising tighter algorithms; (ii) When the bounds are tight, the instances on which the bounds were derived can help to determine under which additional assumptions the lower bounds do not hold. Such assumptions might allow us to derive improved upper bounds; (iii) When we can control some parameters of the problem, e.g., the number of arms in an action, their effect on the lower bound can help us tune them for each application. 

In this work, we derive problem-dependent (\textbf{Theorem \ref{theorem:dependent_lower_bound}}) and problem-independent (\textbf{Theorem \ref{theorem:independent_lower_bound}}) lower bounds that hold for general reward functions under mild assumptions. The problem-dependent bound shows that for any `good' bandit strategy, there exists a CMAB instance such that the asymptotic regret must be larger than a certain logarithmic rate. The problem-independent bound shows that for any strategy and any large enough horizon $T$, there exists a horizon-dependent instance with a $\sqrt{T}$ regret. To derive these bounds, we define a family of action sets for CMAB problems, which we call \emph{$\commset$-disjoint}. There, a subset of arms $\commset$ appear in all actions and independent of other actions, while the rest of the arms appear in a single action. We then prove that for $\commset$-disjoint problems, both bounds depend on a new modified Gini-smoothness measure; specifically, they reproduce existing lower bounds for both the linear reward function and the PMC problem. If the reward function is also monotone, as in most practical applications, we derive an additional bound that depends on the Gini-smoothness of the reward and matches the upper bound of \citep{merlisM19} up to logarithmic factors (\textbf{Proposition \ref{prop:smoothness-relation-euclid}}). Thus, our results demonstrate that without any additional assumptions, the bounds are tight for almost any reward function.

\section{Related Work}
The general framework of combinatorial bandits with semi-bandit feedback was first presented in \citep{chen2013combinatorial}. Since then, it has had many extensions, e.g., for the case of probabilistically-triggered arms, where the set of arms in an action might be random \citep{chen2016combinatorialA,wang2017improving}, and for reward functions that depend on the arm distribution \citep{chen2016combinatorialB}. Moreover, many previous works focus on specific instances of this problem, e.g., linear reward functions \citep{kveton2015tight,combes2015combinatorial,degenne2016combinatorial}, cascading bandits \citep{kveton2015cascading,kveton2015combinatorial} and more. Recently, \citet{merlisM19} presented BC-UCB, a Bernstein-based UCB algorithm with regret bounds that depend on a new smoothness measure, which they call the \emph{Gini-weighted smoothness}. Specifically, they show that by combining the reward nonlinearity with the local behavior of the confidence intervals, the dependency of previous regret bounds in the maximal action size can be removed. 
In this work, we show that for monotone reward functions, the Gini-smoothness also characterizes the lower bounds for CMAB problems and therefore prove that this upper bound is tight. In addition, while all previously stated papers assume that the reward function is monotone, a few papers also support non-monotone reward functions \citep{wang2018thompson,huyuk2019thompson}. We also present lowers bounds for this scenario.

Although there has been extensive work on regret upper bounds for CMABs, there are almost no results on lower bounds for this setting. \citet{kveton2015tight} derived lower bounds for the linear reward function with general arm distributions, and when arms are also independent, lower bounds can be found in \citep{degenne2016combinatorial,combes2015combinatorial}. Also, \citet{kveton2015cascading} derived lower bounds for cascading bandits and \citet{merlisM19} derived bounds for the PMC problem. Nevertheless, and to the best of our knowledge, there are no lower bounds for general reward functions. A comparison of our bounds to previous related bounds can be found in Table \ref{table:comparison}.

In contrast to the CMAB problem, the lower bounds for MABs are well characterized. In their seminal work, \citet{lai1985asymptotically} presented the first general problem-dependent lower bound for MABs, which was later extended by \citet{burnetas1996optimal}. In terms of problem-independent bounds, \citet{auer2002nonstochastic} derived an $\Omega(\sqrt{KT})$ lower bound for $K$-armed bandit problems with time horizon $T$, whose constants were later improved by \citet{cesa2006prediction}. Also, \citet{mannor2004sample} proved problem-independent lower bounds with both linear and logarithmic regimes. 
Recently, \citet{garivier2018explore} presented a general tool that allows deriving various lower bounds for MABs. We adapt this tool for the CMAB problem to derive our new regret bounds.

\setlength {\tabcolsep}{0.75pt}
\begin{table*}
\centering
\begin{threeparttable}
\caption{Upper (UB) and lower (LB) bounds of different CMAB problems for arbitrary action sets. Dep./Ind. are problem-dependent and problem-independent bounds, and the notations follow Section \ref{section:prelim}. $\gamma_\infty$ is the global Lipschitz constant of a reward function, and for the Gini-smoothness $\gamma_g$, it holds that $\gamma_g\!\le\!\sqrt{\Nbatch}\gamma_\infty$ \citep{merlisM19}. $\dr{\min}$ is the minimal gap.}
\label{table:comparison}
{\footnotesize 
\begin{tabular}{|c|c|c|c|>{\columncolor[gray]{0.9}}c| >{\columncolor[gray]{0.9}}c|}\hline
 \textbf{CMAB problem} &  \textbf{Type} & \textbf{Previous UB} & \textbf{Previous LB} & \textbf{Theorem \ref{theorem:dependent_lower_bound} or \ref{theorem:independent_lower_bound}} & \textbf{Proposition \ref{prop:smoothness-relation-euclid}} \\ \hline 
 
  \multirow{2}{3.85cm}[-0.225cm]{\centering\textit{General reward functions}} &
  Dep.  & 
  $\Ocal\br*{\frac{\gamma_\infty^2\Narms\Nbatch\ln T}{\dr{\min}}}^\dag$&
  None&
  $\Omega\br*{\max_{\meanvec,
  \commset}\frac{\tilde{\gamma}_g^2(\meanvec;\commset)\Narms\ln T}{\dr{\min}\Ndis}}$&
  NA\\ 
  \hhline{~-----}
   & 
   Ind. & 
   None & 
   None &
   $\Omega\br*{\max_{\meanvec,
  \commset}\sqrt{\frac{\tilde{\gamma}_g^2(\meanvec;\commset)\Narms T}{\Ndis}}}$&
  NA\\ 
  \hline
  \multirow{2}{3.85cm}[-0.25cm]{\centering\textit{Monotone reward functions}}  &
  Dep.  &
  $\Ocal\br*{\frac{\gamma_g^2\Narms\ln^2\Nbatch\ln T}{\dr{\min}}}^\ddag$&
  None &
  $\Omega\br*{\max_{\meanvec,
  \commset}\frac{\tilde{\gamma}_g^2(\meanvec;\commset)\Narms\ln T}{\dr{\min}\Ndis}}$&
  $\tilde\Omega\br*{\frac{\gamma_g^2\Narms\ln T}{\dr{\min}}}$ \\ 
  \hhline{~-----}
   &  
   Ind. & 
   $\Ocal\br*{\gamma_g\ln\Nbatch\sqrt{\Narms T}}^\ddag$&
   None &
   $\Omega\br*{\max_{\meanvec,
  \commset}\sqrt{\frac{\tilde{\gamma}_g^2(\meanvec;\commset)\Narms T}{\Ndis}}}$&
  $\tilde\Omega\br*{\gamma_g\sqrt{\Narms T}}$\\ 
  \hline
  \multirow{2}{3.85cm}[0.1cm]{\centering\textit{Linear reward function} \vspace{-0.1cm}{\footnotesize $$\rdef=\sum_{i\in\action}\meanval_i$$} }  &
  Dep.  &
  $\Ocal\br*{\frac{\Narms\Nbatch\ln T}{\dr{\min}}}^\S$ &
  $\Omega\br*{\frac{\Narms\Nbatch\ln T}{\dr{\min}}}^\S$ &
  $\Omega\br*{\frac{\Narms\Nbatch\ln T}{\dr{\min}}}$ & $\Omega\br*{\frac{\Narms\Nbatch\ln T}{\br*{\ln\Nbatch}\dr{\min}}}$ \\ 
  \hhline{~-----}
   &  
   Ind. & 
   $\Ocal\br*{\sqrt{\Narms\Nbatch T}}^\S$  &
   $\Omega\br*{\sqrt{\Narms\Nbatch T}}^\S$ &
   $\Omega\br*{\sqrt{\Narms\Nbatch T}}$ & $\Omega\br*{\sqrt{\frac{\Narms\Nbatch T}{\ln\Nbatch}}}$\\ 
  \hline
  \multirow{2}{3.85cm}[0.1cm]{\centering\textit{PMC problem} \vspace{-0.18cm}{\mathTable $$\rdef\!=\!\sum_{i=1}^\Nitems\!\br*{\!\!1\!-\!\prod_{j\in\action}\!(1-\meanval_{ij})\!\!}$$} }  &
  Dep.  &
  $\Ocal\br*{\frac{\Narms\Nitems^2\ln^2\Nbatch\ln T}{\dr{\min}}}^\ddag$ &
  $\Omega\br*{\frac{\Narms\Nitems^2\ln T}{\dr{\min}}}^\ddag$ &
  $\Omega\br*{\frac{\Narms\Nitems^2\ln T}{\dr{\min}}}$ & $\Omega\br*{\frac{\Narms\Nitems^2\ln T}{\br*{\ln\Nbatch}^2\dr{\min}}}$ \\ [0.2cm]
  \hhline{~-----}
   &  
   Ind. & 
   $\Ocal\br*{\Nitems\ln\Nbatch\sqrt{\Narms T}}^\ddag$  &
   $\Omega\br*{\Nitems\sqrt{\Narms T}}^\ddag$ &
   $\Omega\br*{\Nitems\sqrt{\Narms T}}$ & $\Omega\br*{\frac{\Nitems\sqrt{\Narms T}}{\ln\Nbatch}}$\\ [0.2cm]
  \hline
\end{tabular}
}\vspace{-0.075cm}
\begin{tablenotes}
      \small
      \item \hspace{-0.65cm} $^\dag$\citep{wang2018thompson}, requires independent arms.
      \,\,  $^\ddag$\citep{merlisM19}
      \,\,  $^\S$\citep{kveton2015tight}
\end{tablenotes}
\vspace{-0.3cm}
\end{threeparttable}
\end{table*}
\setlength{\tabcolsep}{6pt}

\section{Preliminaries and Notations}
\label{section:prelim}

We start with some notations. Let  $\brs{n}=\brc*{1,\dots,n}$, and for any vector $\xb\!\in\!\R^n$ and set $I\!\subset\! \brs*{n}$, denote by $\xb_I$, a sub-vector of $\xb$ that contains only elements from $I$. We denote the Kullback-Leibler (KL) divergence between two distributions $\unu,\unu'$ by $\kl(\unu,\unu')$, and the KL divergence between two Bernoulli random variables with expectations $p,q$ by $\klBin(p,q)$. For any vector $\xb\!\in\!\R^n$, let $\xb^s$ be a permutation such that $x_1^s\!\le\! \dots\!\le\! x_n^s$, and define the increasing permutation of vector $\xb\!\in\!\R^n$ w.r.t. a set $I$ as $p^{\xb,\commset}\!=\!\brs*{\xb_{I^c}^s,\xb_I}\!\in\!\R^n$; namely, the beginning of the vector $p^{\xb,\commset}$ contains a sorted permutation of the elements of $\xb$ in $I^c=\brs*{n}/I$, and its end contains the elements of $\xb$ in $I$. Finally, for any set $I$ of bounded size $\abs*{I}\le\Nbatch$, we denote by $\Ndis\!=\!\Nbatch-\abs{I}$ the size of the complementary set w.r.t. $\Nbatch$.

We work under the combinatorial multi-armed bandit setting with semi-bandit feedback. Denote the number of arms (`base arms') by $\Narms$, and let $\actionset\subset 2^{\brs*{\Narms}}$ be the set of possible actions (`super arms'), that is, the set that contains all valid combinations of base arms that the agent can choose. The number of base arms in each action $\action\in\actionset$ is bounded by $\abs{\action}\le\Nbatch$, and w.l.o.g., assume that $\abs{\action}=\Nbatch$. At the beginning of each round $t$, the arms generate an observation vector $\obs{t}\!=\!\br*{\obs[1]{t},\dots,\obs[\Narms]{t}}\in\brs*{0,1}^\Narms$, sampled from a fixed distribution independently of other rounds. Then, the agent chooses an action $\action_t\in\actionset$ and observes feedback $\obs[\action]{t}\triangleq\brc*{\br*{i,\obs[i]{t}}, \forall i\in\action_t}$.
Denote the means of base arms by $\E\brs*{\obs{t}}\!=\!\meanvec\!=\!\br*{\meanval_1,\dots,\meanval_{\Narms}}$. The goal of the agent is to maximize a known reward function $\rdef$, without knowing $\meanvec$. Specifically, the agent aims to minimize its regret $R(T) = \sum_{t=1}^T\br*{\reward{\action^*}{\meanvec}-\reward{\action_t}{\meanvec}}\triangleq\sum_{t=1}^T\dr{\action_t}$, where $\action^*\in\argmax_{\action\in\actionset}\rdef$ is an optimal action\footnote{Previous work on regret upper bounds also allows approximate maximization of $r$. We focus on the best achievable performance, so we assume we can efficiently maximize $r$.} and $\dr{\action_t}=\reward{\action^*}{\meanvec}-\reward{\action_t}{\meanvec}$ is the suboptimality gap of $\action_t$.
To prove the lower bounds, we require a mild assumption on the reward function, which we call \emph{index invariance}:
\begin{definition}
\label{assum: differentiable reward}
A reward function $\rdef:\actionset\times\brs*{0,1}^\Narms\to\R$ is called differentiable if for any $\action\in\actionset$, it is differentiable in $\meanvec\in\brs*{0,1}^\Narms$. 
\end{definition}
\begin{definition}
\label{assum: reward function}
A differentiable reward function $\rdef:\actionset\times\brs*{0,1}^\Narms\to\R$ is called smooth index invariant if for any $\action\in\actionset$, it only depends on the arms in $\action$, i.e., $\rdef=\rewardvec{\meanvec_\action}$. 
\end{definition}

When the function is index invariant, and with a slight abuse of notations, we also write $\rewardvec{\meanvec}$, with $\meanvec\in\R^{\Nbatch}$, to represent the mean of arms $\meanvec_\action$ for $\abs{\action}=\Nbatch$. This assumption helps avoiding cases in which specific arms behave inherently different than other arms, such that the problem becomes much easier. For example, for the biased linear function $\rdef = \sum_{i\in \action} \br*{\meanval_i +\Narms i}$ and for any $\meanvec\in\brs*{0,1}^\Narms$, the optimal action is $\action^*=\argmax_{\action\in\actionset}\sum_{i\in\action} i$, regardless of the arm means; therefore, both the upper and lower bounds for this reward function trivially equal zero. In contrast, the lower bound for the linear function are nonzero (see Table \ref{table:comparison}); thus, without the index-invariance, the lower bounds cannot be characterized solely by the gradient of reward function w.r.t. $\meanvec$, in contrast to the existing upper bounds.
To the best of our knowledge, all practical applications for CMABs are index-invariant or can be written as a sum over an index-invariant function that is applied on different arms (e.g., as in Corollary \ref{corollary: sum index invariant dependent}). We also believe that our analysis will hold for reward functions that depend on the \emph{order} of arms inside an action. However, we leave this extension for future work. 
Besides this assumption, we later move our focus to monotone reward functions, which are defined as follows:
\begin{definition}
\label{def:monotone}
A differential reward function $\rdef:\actionset\times\brs*{0,1}^\Narms\to\R$ is called monotone if for any $\action\in\actionset$, any $\meanvec\in\brs*{0,1}^\Narms$ and any $i\in\brs*{\Narms}$, it holds that $\nabla_i \rdef\ge0$.
\end{definition}
We remark that in most previous work, the upper bounds only hold for monotone functions, which include most of the practical application, e.g., the linear and PMC problems. 
We end this part of the preliminaries with an important inequality that was derived for MABs and will enable us to derive our new bounds for CMABs. Let $\brs*{\Narms}$ be a set of arms, where each arm $a\in\brs*{\Narms}$ is characterized by a distribution $\nu_a$ over $\R^\Nbatch$, and denote $\unu=\brc*{\nu_a}_{a\in\brs*{\Narms}}$.\footnote{\citet{garivier2018explore} assume that $\nu_a$ are distributions over $\R$, but the exact same proof holds for distributions over $\R^\Nbatch$.} Assume that at each round, when playing $a_t$, a sample $Y_t$ is drawn independently at random from $\nu_{a_t}$. Let $\psi$ be a strategy that chooses an arm according to the history and internal i.i.d randomization $U_t\in\brs*{0,1}$. Namely, if $H_t = \br*{U_0,Y_1,U_1,\dots,U_t,Y_t}$, then  $a_{t+1}=\psi_t(H_t)$. Also, let $N_{\psi,a}(T)$ be the number of times an arm $a$ was played under strategy $\psi$ up to time $T$. Under these notations, the following holds:
\begin{lemma}[\citealt{garivier2018explore}]
\label{lemma:kl count bound}
 For all bandit problems $\unu,\unu'$, for all $\sigma(H_T)$-measurable random variables $Z$ with values in $[0,1]$, 
\begin{align} 
    \label{eq:kl count bound}
    \sum_{a=1}^\Narms \E_{\unu}\brs*{N_{\psi,a}(T)}\kl(\nu_a,\nu_a') \ge \klBin\br*{\E_\unu\brs*{Z},\E_{\unu'}\brs*{Z}}\, ,
\end{align}
where $\klBin(p,q)=p\ln\frac{p}{q}+(1-p)\ln\frac{1-p}{1-q}$.
\end{lemma}
In the combinatorial case we use similar notations and denote the action counts by $N_{\psi,\action}(T)$.

\subsection{Smoothness Measures}

\begin{figure*}
\centering
\subfigure[\footnotesize Hoeffding confidence bounds]{
\includegraphics[width=0.31\linewidth]{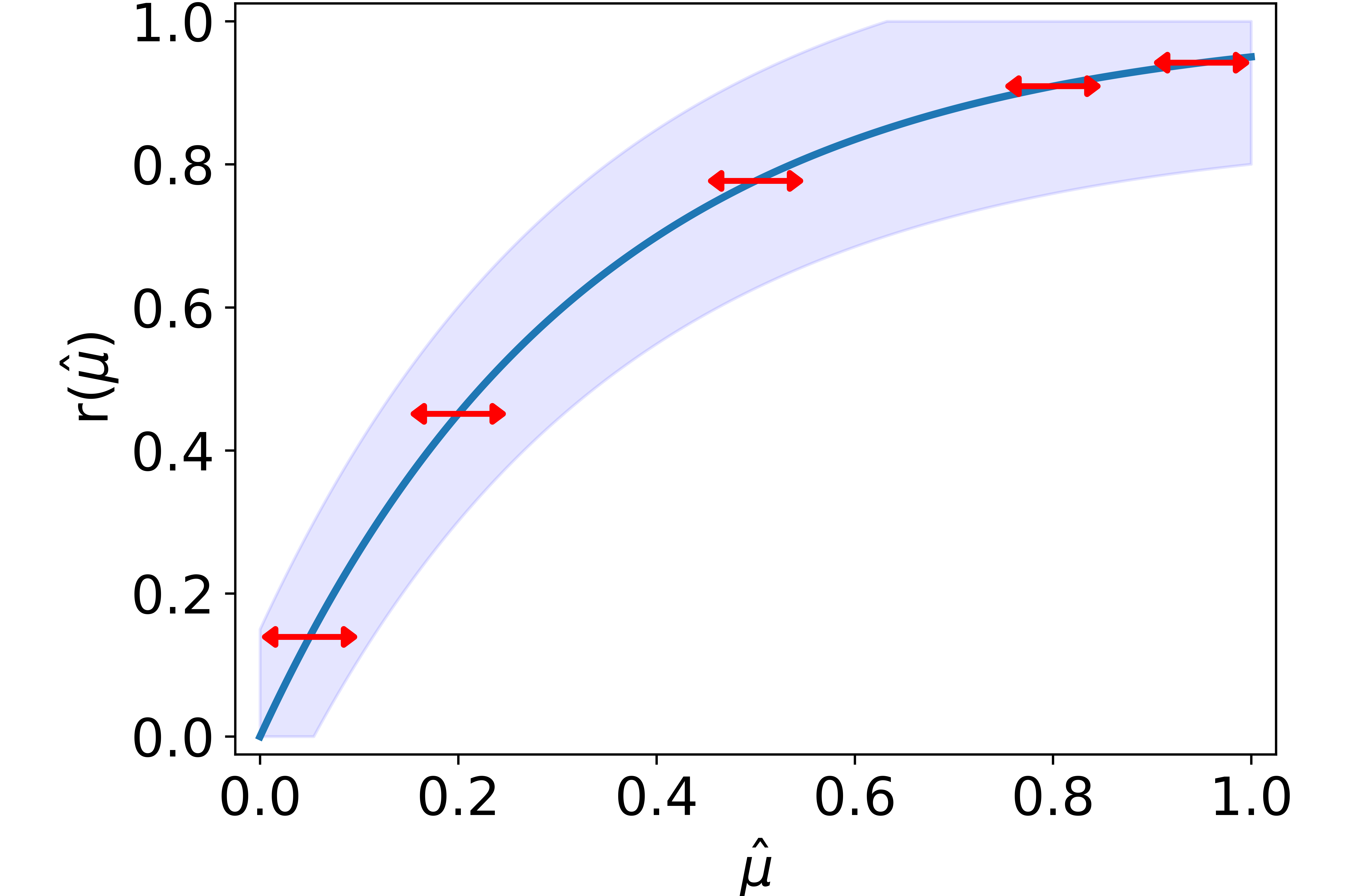} 
\label{subfig:hoeff}} 
\subfigure[\footnotesize Bernstein confidence bounds]{
\includegraphics[width=0.31\linewidth]{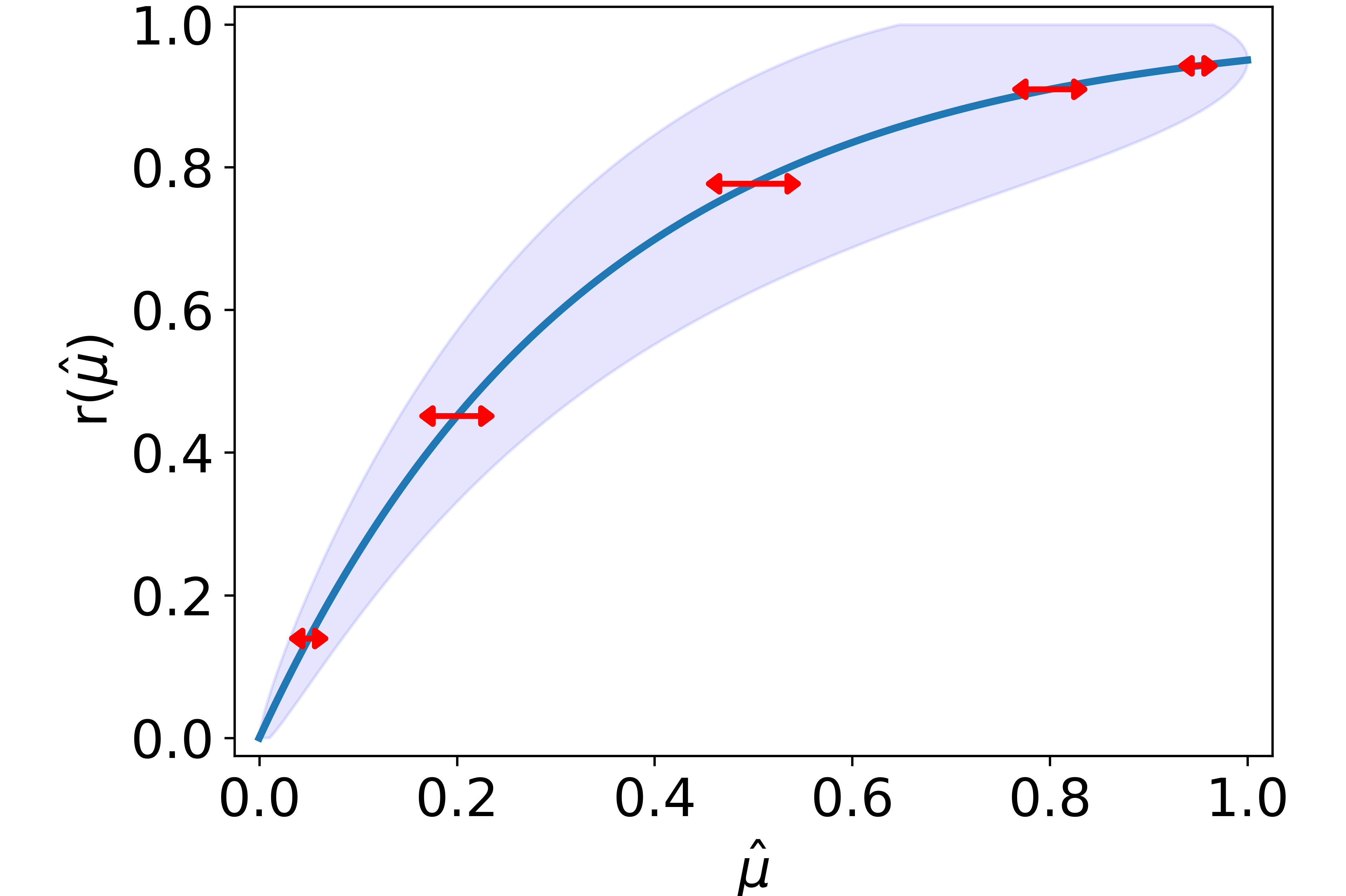} 
\label{subfig:bern-loose}}
\subfigure[\footnotesize  Bernstein confidence bounds]{
\includegraphics[width=0.31\linewidth]{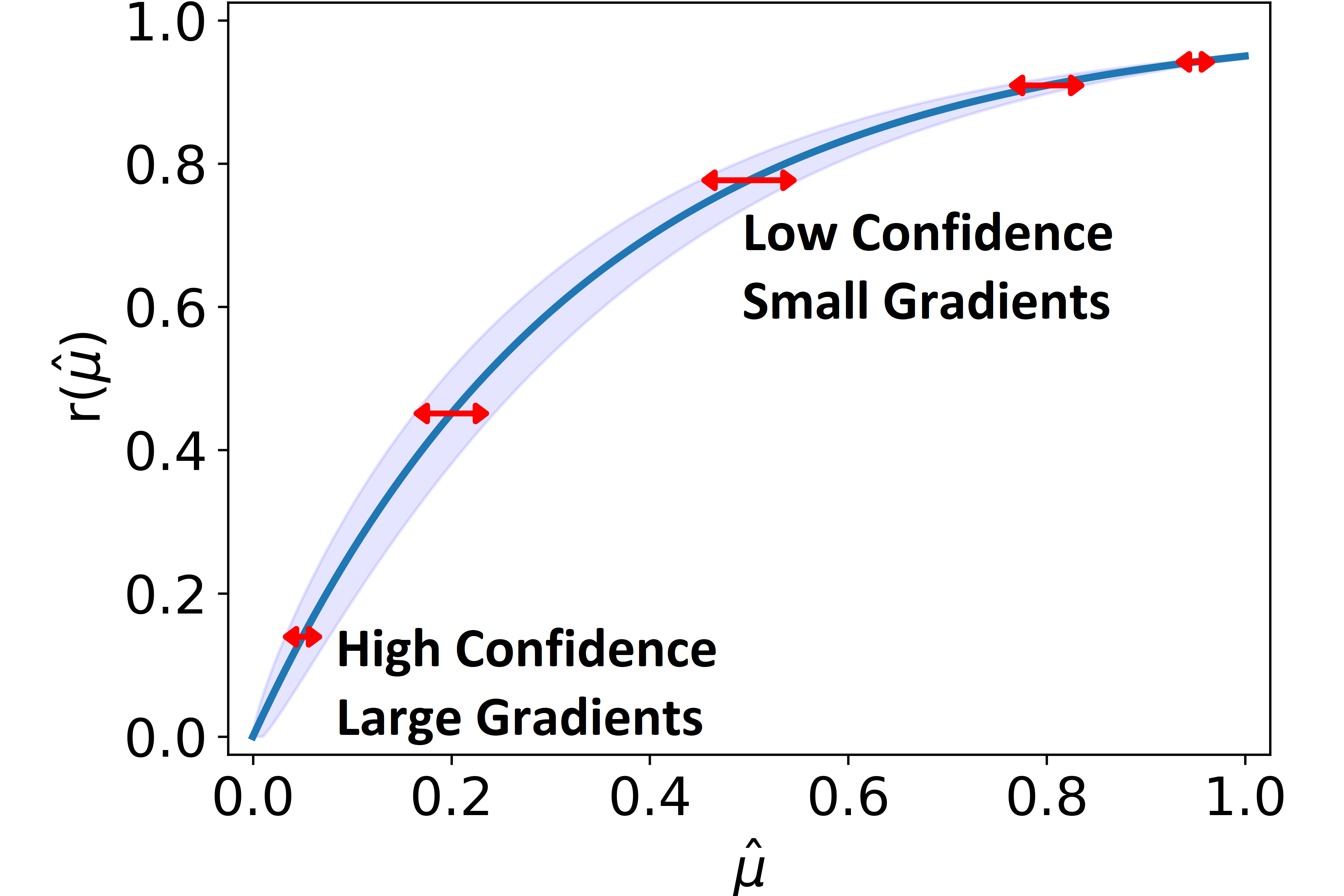}
\label{subfig:bern-tight}}
\caption{Red arrows: confidence intervals on the function parameter (x-axis confidence), due to either Hoeffding or Bernstein inequalities; the latter is tighter near the edges. Blue zone: the resulting confidence intervals on the reward function in the bold curve (y-axis confidence). In Figures \ref{subfig:hoeff},\ref{subfig:bern-loose}, these intervals are derived using the global Lipschitz constant of the reward $\gamma_\infty$, i.e., $CI\br*{\rewardvec{\hat\meanvec}}\!\lesssim\! \gamma_\infty \sum_i CI(\hat\meanval_i)$. In Figure \ref{subfig:bern-tight}, we present the real confidence interval on the reward, that is much tighter than the bound due to $\gamma_\infty$. This is since the bound is tight where the gradient is large, which is around the edge of the domain, but loose in other areas, where the gradient is small.
} 
\vspace{-0.2cm}
\label{figure:gini-smoothness}
\end{figure*}

We now present the smoothness measures for smooth index-invariant reward functions that govern our lower bounds. The measures are defined for arm parameters $\meanvec\!\in\!\R^\Nbatch$ and a set $\commset\!\subset\!\brs*{\Narms}$ as follows:
\begin{enumerate}
    \item L2 Gini-weighted smoothness \vspace{-0.25cm}
    {\small
    \begin{align}  
    \label{eq:smoothDef L2}
    \gamma_{g,2}^2(\meanvec;\commset) 
    &\triangleq\sum_{i=1}^\Ndis p^{\meanvec,\commset}_i(1-p^{\meanvec,\commset}_i)\nabla_i\rewardvec{p^{\meanvec,\commset}}^2 
    = \sum_{i\notin \commset} \meanval_i(1-\meanval_i)\nabla_i\rewardvec{\meanvec}^2
    \end{align}
    }\vspace{-0.5cm}
    \item L1 Gini-weighted smoothness\vspace{-0.25cm}
    {\small
    \begin{align}  
    \label{eq:smoothDef L1}
    \gamma_{g,1}^2(\meanvec;\commset) 
    &\triangleq\br*{\sum_{i=1}^{\Ndis}\sqrt{p^{\meanvec,\commset}_i(1-p^{\meanvec,\commset}_i)}\nabla_i\rewardvec{p^{\meanvec,\commset}}}^2
    = \br*{\sum_{i\notin \commset} \sqrt{\meanval_i(1-\meanval_i)}\nabla_i\rewardvec{\meanvec}}^2
    \end{align}
    }\vspace{-0.5cm}
    \item \emph{Modified} Gini-weighted smoothness\vspace{-0.25cm}
    {\small
    \begin{align}
    \label{eq:newSmooth}
       &\tilde{\gamma}_g^2(\meanvec;\commset) = \sum_{i=1}^\Ndis p^{\meanvec,\commset}_i(1-p^{\meanvec,\commset}_i)\nabla_i\rewardvec{p^{\meanvec,\commset}}^2 
       +2\sum_{i=1}^\Ndis \sum_{j=i+1}^\Ndis p^{\meanvec,\commset}_i(1-p^{\meanvec,\commset}_j)\nabla_i\rewardvec{p^{\meanvec,\commset}}\nabla_j\rewardvec{p^{\meanvec,\commset}} 
    \end{align}
    }\vspace{-0.5cm}
\end{enumerate}
Notice that the last equality in Equations \eqref{eq:smoothDef L2},\eqref{eq:smoothDef L1} is due to the index-invariance assumption. Specifically, the assumption implies that $\nabla_i\rewardvec{\meanvec}$ only depends on $\meanval_i$ and the set of values $\brc*{\meanval_i}_{i=1}^\Narms$. 
For smooth index-invariant reward functions, $\gamma_{g,2}^2(\meanvec;\commset)$ is an extension of the Gini-smoothness, as defined in \citep{merlisM19}. In particular, their smoothness parameter is defined as $\gamma_g=\max_{\commset,\meanvec}\gamma_{g,2}(\meanvec;\commset)=\max_{\meanvec}\gamma_{g,2}(\meanvec;\emptyset)$. A motivation to this smoothness criterion is presented in Figure \ref{figure:gini-smoothness}. Notably, the performance of any algorithm strongly depends on the uncertainty in the reward of actions. However, algorithms only have access to uncertainty in arm parameters, and arm uncertainty must be translated into reward uncertainty. The figure illustrates that doing so using the global Lipschitz constant might lead to loose bounds. Intuitively, if the gradients are small, even wide confidence intervals do not cause high uncertainty in the reward. Similarly, narrow confidence intervals do not lead to high reward uncertainty even where the gradients are large. Thus, \citet{merlisM19} suggested weighting the gradients according to the confidence intervals of the arms. Specifically, their algorithm (BC-UCB) relies on Empirical-Bernstein concentration-bounds \citep{audibert2009exploration}, that depend on the variance of the arms and are proportional to $\sqrt{\meanval_i(1-\meanval_i)}$ for Bernoulli arms. Similarly weighting the gradients leads to the Gini-smoothness measures.

In the following sections, we prove lower bounds for the CMAB problem that depend on $\tilde{\gamma}_g^2(\meanvec;\commset)$. While complex at first glance, $\tilde{\gamma}_g^2(\meanvec;\commset)$ is actually closely related to the other smoothness measures. Notably, observe that the only difference between $\tilde{\gamma}_g^2(\meanvec;\commset)$ and $\gamma_{g,1}^2(\meanvec;\commset)$ is a small modification to the second (cross) term of \eqref{eq:newSmooth}. In Proposition \ref{prop:smoothness-relation-modified-L1}, we indeed prove that $\tilde{\gamma}_g^2(\meanvec;\commset)\ge\tilde\Omega(\gamma_{g,1}^2(\meanvec;\commset))$. When the function is monotone, we later prove that $\gamma_{g,1}(\meanvec;\commset)$ can also be related to $\gamma_{g,2}(\meanvec;\emptyset)$, which leads to tight lower bounds, up to logarithmic factors.

\section{Problem-Dependent Lower Bounds}

In this section, we prove a problem-dependent lower bound. Specifically, we show that there exist a CMAB instance, such that the asymptotic regret of any consistent strategy on this instance is lower bounded by a logarithmic term that depends on $\tilde{\gamma}_g(\meanvec;\commset)$ and the minimal gap $\dr{}=\min_{\action\in\actionset,\dr{\action}>0}\dr{\action}$. Consistent strategy is defined as follows:
\begin{definition}
A bandit strategy $\psi$ is called consistent if for any CMAB problem, any $\action \in\actionset$ such that $\dr{\action}>0$ and any $0<\alpha\le1$, it holds that $\E\brs*{N_{\psi,\action}(T)}=o\br*{T^\alpha}$.
\end{definition}
To prove the lower bounds, we focus on a subset of CMAB problems which we call \emph{$I$-disjoint}.
\begin{definition}
\label{def:I-disjoint}
For a given subset $\commset\!\subset\!\brs*{\Narms}$, a CMAB problem is called $\commset$-disjoint if all arms $i\!\in\!\commset$ are mutually independent of all arms $i\!\notin\!\commset$ and also $\action_1\!\cap\!\action_2\!=\!\commset$ for any $\action_1\!\ne\!\action_2\!\in\!\actionset$.
\end{definition}
Since $\abs{\action}\le\Nbatch$, we implicitly assume that $\abs{\commset}\le\Nbatch$ and denote the effective maximal action size by $\Ndis=\Nbatch-\abs{I}$. In $\commset$-disjoint CMAB problems, the base arms $i\in\commset$ appear in all actions and are mutually independent of the other arms. The rest of the arms can only appear in one action. This notion of CMAB problems actually extends the action sets from previous work -- for the linear reward function, \citet{kveton2015tight} divided the arms into $\frac{\Narms}{\Nbatch}$ disjoint groups, which is equivalent to $\commset\!=\!\emptyset$. In contrast, in the CMAB instance on which the PMC lower bounds were derived, only a single arm per item varied between actions \citep{merlisM19}. If there are $\Nitems$ items, this is equivalent to $\commset\!=\!\brs*{\Nbatch-M}$. We show that choosing the `worst-case' set $\commset$ naturally results with tighter lower bounds. We start by deriving general problem-dependent lower bound, using  Lemma \ref{lemma:kl count bound}:

\begin{restatable}{lemma-rst}{generalDependentBound}
\label{lemma:generalDependentBound}
Let $\rewardNoArgs$ be a smooth index invariant reward function with $\abs*{\action}=\Nbatch$ for all $\action\in\actionset$. Also, let $\unu$ be the action distribution of an $I$-disjoint CMAB problem such that there exists an arm $i\notin\commset$ in $\action^*$ with $\Pr\brc*{\obs[i]{t}=\meanval_i}<1$ and $\nabla_i\reward{\action^*}{\meanvec}\ne0$. 
Then, for all consistent strategies $\psi$ and all suboptimal actions $S$,
{\small
\begin{align*}
    \liminf_{T\to\infty}\frac{\E_\unu\brs*{N_{\psi,\action}(T)}}{\ln T} \ge \frac{1}{\kl(\nu_\action,\nu_{\action^*})} \enspace .
\end{align*}
}
Specifically, it holds that
{\small
\begin{align*}
    \liminf_{T\to\infty}\frac{\E_\unu\brs*{R(T)}}{\ln T} \ge \sum_{\action:\dr{\action}>0}\frac{\dr{\action}}{\kl(\nu_\action,\nu_{\action^*})}\enspace .
\end{align*}
}
\end{restatable}

The proof is in Appendix \ref{appendix:general-bounds}, and partially follows Theorem 1 of \citep{garivier2018explore}, with some adjustments due to the nonlinearity of the reward function. Although this lemma gives a general lower bound for $\commset$-disjoint CMAB problems, it has no clear dependence on any smoothness measure of the reward. To derive lower bounds that directly depend on such measures, we carefully design the arm distributions $\nu_\action, \nu_{\action^*}$ and analyze both $\dr{\action}$ and $\kl(\nu_\action,\nu_{\action^*})$ for these distributions. We do so in the following theorem, which results with the desired lower bound:

\begin{myTheorem}
\label{theorem:dependent_lower_bound}
Let $\rewardNoArgs$ be a smooth index invariant reward function. 
For any $\meanvec\in\brs*{0,1}^\Nbatch$ and any small enough $\dr{}>0$, there exists an instance of an $\commset$-disjoint bandit problem $\unu$ with minimal gap $\dr{}$ and $\E\brs*{\nu_{\action^*}}=\meanvec$, such that the expected regret of any consistent algorithm is bounded by 
{\small
\begin{align*}
\liminf_{T\to\infty} \frac{\E_\unu\brs*{R(t)}}{\ln T} 
\ge \max_\commset \frac{(\Narms-2\Nbatch)\tilde{\gamma}_g^2(\meanvec;\commset)}{8\Ndis\dr{}} \triangleq DB^*_r(\dr{};\meanvec)\enspace.
\end{align*}
}
\end{myTheorem}
The result can also be easily extended as follows:
\begin{corollary}
\label{corollary: sum index invariant dependent}
Let $\rewardNoArgs$ be a smooth index invariant reward function and let $\tilde\meanvec\!=\!\brs*{\meanvec_1,\dots,\meanvec_\Nitems}\!\in\!\brs*{0,1}^{\Nitems\Nbatch}$, with $\meanvec_i\in\brs*{0,1}^\Nbatch$ for all $i\in\brs*{\Nitems}$. Also, define $\tilde{\rewardNoArgs}\br*{\tilde\meanvec} = \sum_{i=1}^\Nitems \rewardvec{\meanvec_i}$. Then, for any $\meanvec\in\brs*{0,1}^\Nbatch$, there exists a CMAB instance that aims to maximize $\tilde{\rewardNoArgs}$ and has a minimal gap $\dr{}$, such that the optimal action has means $\meanvec_i=\meanvec$ for all $i\in\brs*{\Nitems}$ and the expected regret of any consistent algorithm is bounded by 
{\small
\begin{align*}
\liminf_{T\to\infty} \frac{\E_\unu\brs*{R(t)}}{\ln T} 
\ge \Nitems^2\cdot DB^*_r(\dr{};\meanvec)\enspace.
\end{align*}
}
\end{corollary}
Notice that for each summand of $\tilde{\rewardNoArgs}$, $\Nbatch$ arms are chosen from a total of $\Narms$ arms, independently of all other summands. Therefore, the problem can also be formulated as an $\tilde\Narms=\Nitems\Narms$-armed problem, with $\tilde\Nbatch=\Nitems\Nbatch$ selected arms per round. The two formulations are equivalent, and we decided to follow this notation for consistancy with \citep{merlisM19}. The corollary is a direct result of Theorem \ref{theorem:dependent_lower_bound}; specifically, by fixing the arm distribution to be identical for all of the summands of $\tilde{\rewardNoArgs}$, we get $\tilde{\rewardNoArgs}=\Nitems\cdot \rewardNoArgs$, and since the Gini-smoothness parameters linearly scale with the reward function, the bound naturally follows.

Before proving the theorem, we start with a short discussion on the tightness of the results and the relation to existing upper bounds. 
One interesting case is when $p^{\meanvec,\commset}_i=p_0$ for all $i\in\brs*{\Ndis}$. Due to the index invariance, the gradient components are also equal, which results with $\tilde{\gamma}_g^2(\meanvec;\commset)=\Ndis^2p_0(1-p_0)\nabla_i\rewardvec{p^{\meanvec,\commset}}^2$. This choice allows us to easily reproduce the existing lower bounds and leads to the bounds stated in Table \ref{table:comparison}. Specifically, for the linear reward function, we achieve the bound by choosing $\commset\!=\!\emptyset$ and $p_0\!=\!\frac{1}{2}$; for the PMC problem, we start by writing $\tilde{\rewardNoArgs}=\sum_{i=1}^\Nitems \reward{\action}{\meanvec_i}$ for $\rdef=1-\prod_{j\in\action}\meanval_{j}$. Then, we apply Theorem $\ref{theorem:dependent_lower_bound}$ on $\rewardNoArgs$ and choose the subset $\commset$ such that it contains all arms except for a single (first) element from each vector $\meanvec_i=\br*{\frac{1}{2},0,\dots,0}$, i.e., $\Ndis=1$ and $p_0=\frac{1}{2}$. The lower bound is then a direct result of Corollary \ref{corollary: sum index invariant dependent}. We remark that although choosing $\Ndis=\Nbatch$ is seemingly optimal, this is not always the case. A notable example for this issue is the reward function $\rdef = 1-e^{-\sum_{i\in\action} \meanval_i^2}$. For this instance, optimizing over $p_0$ leads to a bound of $\tilde{\gamma}_g^2(\meanvec;\commset)/\Ndis=\Omega\br*{1/\sqrt{\Ndis}}$, and the optimal choice is $\Ndis=\Ocal(1)$. In this example it also holds that $\gamma_g=\Ocal(1)$, and, therefore, this bound is tight.

We note that this bound only requires the smooth index invariance assumption and holds for non-monotone reward functions. When the reward is also monotone, observe that $\tilde{\gamma}_g^2(\meanvec;\commset)\ge\gamma_{g,2}^2(\meanvec;\commset)$. Choosing $\commset=\emptyset$ and maximizing over $\meanvec$ leads to a lower bound of $\Omega\br*{\frac{\Narms\gamma_g^2\ln T}{\dr{}\Nbatch}}$, which differs from the upper bound of \citep{merlisM19} by a factor of $\Nbatch$. We later prove that when the reward is monotone, a stronger lower bound can be derived, such that it matches the upper bound up to logarithmic factors. 

Next, we present the proof of Theorem \ref{theorem:dependent_lower_bound} which is composed of three parts. We first present a carefully designed parametric arm distributions, that allow controlling the suboptimality gap while retaining low KL-divergence. We then bound both the gap and the KL-divergence in terms of the parameters of the distributions and apply Lemma \ref{lemma:generalDependentBound} using these bounds. We conclude the proof by optimizing the resulting lower bound over the distribution parameters. 
\begin{proof}
\par \noindent{\textbf{\scshape{Step 1:}} \textit{Fixing a parametric family of CMAB instances.}}

\noindent Denote by $\commset^*$, the maximizer of $\commset$ in $DB^*_r(\dr{};\meanvec)$, and for brevity, let $p=p^{\meanvec,\commset^*}$ and $\Ndis = \Ndis^*=\Nbatch-\abs*{\commset^*}$. Note that this also implies that $0\le p_1\le\dots\le p_\Ndis \le 1$. In addition, due to the form of the lower bound, we can deduce that $p_1>0$ and $p_\Ndis<1$, since otherwise, we can move arms with means 0 and 1 into $\commset^*$ and strictly increase $DB^*_r(\dr{};\meanvec)$. For now, we also assume that no two arms have the same mean, i.e., $0<p_1<\dots<p_\Ndis<1$, and will return to this assumption later in the proof. Without loss of generality, we also assume that $\Ndis>0$ and there exists $i\in\brs*{\Ndis}$ such that $\nabla_i\rewardvec{p}\!\ne\!0$ since otherwise, $DB^*_r(\dr{};\meanvec)\!=\!0$ and trivially holds. Similarly, we assume that $\Narms\!>\!2\Nbatch$.

\setlength{\tabcolsep}{1pt}
\begin{wraptable}{r}{0.55\textwidth} \captionsetup{width=.9\linewidth}
\caption{The probability distributions $\nu$ and $\nu^*$ of arms outside $\commset^*$, with parameters $0\!<\!p_1\!<\!p_2\!<\!\dots\!<\!p_\Ndis\!<\!1$ and $\epsilon\in\R^\Ndis$. Arms in $\commset^*$ have mean $\meanvec_{\commset^*}$ and are independent of the rest of the arms.} \vspace{-0.25cm}
\label{table:lower_bound_distribution}
\begin{center} 
{\footnotesize
\begin{tabular}{|c|c|c|}\hline
  Observation vector                               & Probability of        & Probability of \\ 
  $\obs{t}\!\!=\!\br*{\obs[1]{t},\cdots,\obs[\Ndis]{t}}$& $\obs{t}$ in $\nu^*$& $\obs{t}$ in $\nu$\\ 
  \hline 
  $(1,1,\cdots,1,1)$ & $p_1$ & $p_1-\epsilon_1$\\ 
  \hline
  $(0,1,\cdots,1,1)$ & $p_2-p_1$ & $p_2-p_1-\epsilon_2$\\ 
 \hline
 $\cdots$ & $\cdots$ & $\cdots$\\ 
 \hline
 $(0,0,\cdots,0,1)$ & $p_\Ndis - p_{\Ndis-1}$ & $p_\Ndis - p_{\Ndis-1}-\epsilon_\Ndis$\\ 
 \hline
 $(0,0,\cdots,0,0)$ & $1-p_\Ndis$ & $1-p_\Ndis + \sum_{i=j}^\Ndis\epsilon_j$\\
 \hline
\end{tabular}
}
\end{center} \vspace{-0.5cm}
\end{wraptable}
\setlength{\tabcolsep}{6pt}

We fix the CMAB problem to be $\commset^*$-disjoint, and choose the action set $\actionset$ to be the maximal action set with action sizes $\Nbatch$, i.e., $\abs{\actionset} = \floor*{\frac{\Narms-\abs{\commset^*}}{\Nbatch-\abs{\commset^*}}} \ge \frac{\Narms-\Nbatch}{\Ndis}$. Denote the arm distribution in this problem by $\unu$, and fix the distribution of the common arms $i\in\commset^*$ to any distribution with expectation $\E\brs*{\nu_{\commset^*}}=\meanvec_{\commset^*}$, as long as they are mutually independent of the rest of the arms. For a single action, we set the distribution to be $\nu^*$ with mean $\meanvec^{\nu^*}\triangleq\meanvec$, and for the rest of the actions, we fix it to distribution $\nu$ with mean $\meanvec^\nu$. Both distributions are stated in Table \ref{table:lower_bound_distribution}. The distribution $\nu$  depends on $\epsilon\in\R^\Ndis$, that will be determined later such that $\nu$ is strictly suboptimal, and we denote its suboptimality gap by $\dr{\epsilon}=\rewardvec{\meanvec^{\nu^*}} - \rewardvec{\meanvec^\nu}$. We remark that for $0\!<\!p_1\!<\!\dots\!<\!p_\Ndis\!<\!1$, there exists $b_{\epsilon,0}$ such that $\nu$ is a valid probability distribution for all $\norm*{\epsilon}_\infty\le b_{\epsilon,0}$. We enforce this condition later in the proof.

Note that for all $i\notin\commset^*$, the arms are Bernoulli random variables with mean $\mu_i\notin\brc*{0,1}$, and therefore $\Pr\brc*{\obs[i]{t}=p_i}=0$. Also, since the gradient is not zero for some $i\notin\commset^*$, the conditions of Lemma \ref{lemma:generalDependentBound} hold, and we can bound the regret by
{\small
\begin{align}
\label{eq:lower_bound_form}
    \liminf_{T\to\infty} \frac{\E_\unu\brs*{R(T)}}{\ln T} 
    &\ge \sum_{i=1}^{\abs{\actionset}-1}\frac{\dr{\epsilon}}{\kl(\nu,\nu^*)} 
    \ge \frac{\Narms-2\Nbatch}{\Ndis}\frac{\dr{\epsilon}}{\kl(\nu,\nu^*)} \enspace .
\end{align}
} 
\noindent{\textbf{\scshape{Step 2:}}  \textit{Deriving lower bounds that depend on the distribution parameters $\epsilon$. }}

\noindent Next, we bound both $\kl(\nu,\nu^*)$ and $\dr{\epsilon}$ in terms of $\epsilon$.
\begin{restatable}{lemma-rst}{klBound}\label{lemma:kl-bound} 
Let $p\triangleq p^{\meanvec,\commset}\in\R^\Nbatch$ such that $0<p_1<\dots<p_\Ndis<1$ and define $p_0=0$. Also, let $\nu,\nu^*$ be the distributions stated in Table \ref{table:lower_bound_distribution}. Then, there exists a constant $b_{\epsilon,1}>0$ such that for any $\epsilon\in\R^\Ndis$ with $\norm*{\epsilon}_\infty\le b_{\epsilon,1}$, it holds that
{\small
\begin{align}
\label{eq:kl bound dependent}
    \kl(\nu,\nu^*) 
    \le 2\epsilon^T B_{KL}(p)\epsilon \enspace,
\end{align}
}
where $B_{KL}(p) = D(p)+\frac{1}{1-p_\Ndis}\onevec\onevec^T$, $D(p)\in\R^{\Ndis\times\Ndis}$ is a diagonal matrix whose elements are $D_{ii}(p)=\frac{1}{p_i-p_{i-1}}$ and $\onevec\in\R^\Ndis$ is a vector of ones. 
\end{restatable}
\begin{restatable}{lemma-rst}{gapBound}\label{lemma:gap-bound} 
Let $\rewardNoArgs$ be a smooth index invariant reward function and let $p\in\R^\Nbatch$ such that $0<p_1\le\dots\le p_\Ndis<1$ and there exists $i\in\brs*{\Ndis}$ with $\nabla_i\rewardvec{p}\ne0$. Also, define  $c_j = \sum_{i=j}^\Ndis \nabla_i \rewardvec{p}$ for all $j\in\brs*{\Ndis}$ and let $\ub\in\R^\Ndis$ be a vector such that $c^T\ub>0$. Then, if $\epsilon=\epsilon_0 \ub$, there exists a constant $b_{\epsilon,2}>0$ such that for all $0<\epsilon_0\le b_{\epsilon,2}$, 
{\small
\begin{align}
\label{eq:gap bound dependent}
    \dr{\epsilon} \ge \frac{1}{2}c^T\epsilon>0\enspace.
\end{align}
}
\end{restatable}
The proofs are presented in Appendix \ref{appendix:technical}. Specifically, note that $D_{ii}(p)>0$, and thus $B_{KL}(p)$ is positive definite, and the KL bound equals zero only for $\epsilon=0$.  Also, since there exists $i\in\brs*{\Ndis}$ such that  $\nabla_i\rewardvec{p}\ne0$, it holds that $c\ne0$. When both lemmas hold, substitution into \eqref{eq:lower_bound_form} yields
{\small
\begin{align}
\label{eq:lower_bound_dependent_form}
    \liminf_{T\to\infty} \frac{\E_\unu\brs*{R(T)}}{\ln T}
    & \ge \frac{\Narms-2\Nbatch}{\Ndis}\frac{\dr{\epsilon}}{\kl(\unu,\unu^*)} 
    = \frac{\Narms-2\Nbatch}{\Ndis\dr{\epsilon}}\frac{\dr{\epsilon}^2}{\kl(\unu,\unu^*)} 
    \ge \frac{\Narms-2\Nbatch}{8\Ndis\dr{\epsilon}}\frac{\br*{c^T\epsilon}^2}{\epsilon^TB_{KL}(p)\epsilon} \enspace.
\end{align}
}
\noindent{\textbf{\scshape{Step 3:}} \textit{Finding the worst-case CMAB instance.}}

\noindent We now focus on the function $f_\commset(\epsilon;p) = \frac{\epsilon^Tcc^T\epsilon}{\epsilon^TB_{KL}(p)\epsilon}$, which is defined for any $\epsilon\ne0$, as $B_{KL}(p)$ is positive definite. $B_{KL}(p)$ is also invertible, and we can therefore apply the invertible transformation $\epsilon=B_{KL}^{-1/2}(p)x$, which results with the following function:
{\small
\begin{align*}
    \tilde{f}_\commset(x;p) 
    &= \frac{x^TB_{KL}^{-1/2}(p)cc^TB_{KL}^{-1/2}(p)x}{\norm{x}_2^2} 
    = \frac{\br*{c^TB_{KL}^{-1/2}(p)x}^2}{\norm{x}_2^2} 
    \stackrel{(*)}{\le} \frac{\norm{B_{KL}^{-1/2}(p)c}_2^2\norm{x}_2^2}{\norm{x}_2^2}
    = c^TB_{KL}^{-1}(p)c \enspace,
\end{align*}
}
where $(*)$ is due to Cauchy-Schwarz Inequality, and equality holds for any $\epsilon_0\ne0$ and $x=\epsilon_0 B^{-1/2}_{KL}(p)c$. 
Therefore, the maximal value of $f_\commset(\epsilon;p)$ is $f_\commset(\epsilon^*;p)=c^TB^{-1}_{KL}(p)c$ and can be attained with $\epsilon^*=\epsilon_0 B^{-1}_{KL}(p)c$, for any $\epsilon_0\ne0$. 

Motivated by the maximization property of $\epsilon^*$, we now fix $\epsilon\leftarrow\epsilon^*= \epsilon_0 B_{KL}^{-1}(p)c$, for $0<\epsilon_0 \le b_{\epsilon,2}$ such that $\norm*{\epsilon^*}_\infty\le \min\brc*{b_{\epsilon,0}, b_{\epsilon,1}}$. For this choice, $\epsilon^*\ne0$, since $B_{KL}(p)$ is invertible and $c\ne0$, and $c^TB_{KL}^{-1}(p)c>0$ as required for Lemma \ref{lemma:gap-bound}. We explicitly calculate the lower bound in the following lemma (see proof in Appendix \ref{appendix:technical}):
\begin{restatable}{lemma-rst}{boundExplicitCalc}\label{lemma:boundExplicitCalc} 
Under the notations of Lemmas \ref{lemma:kl-bound} and \ref{lemma:gap-bound}, for any $0\!<\!p_1\!<\!\dots\!<\!p_\Ndis\!<\!1$ and $c\ne0$, if $\epsilon^* = \epsilon_0B_{KL}^{-1}(p)c$, then
{\small
\begin{align*}
     \epsilon^*_i = \epsilon_0\br*{p_i-p_{i-1}}\br*{c_i - \sum_{j=1}^\Ndis\br*{p_j-p_{i-j}}c_j} \enspace.
\end{align*}
}
Also, if $f_\commset(\epsilon;p)\! =\! \frac{\epsilon^Tcc^T\epsilon}{\epsilon^TB_{KL}(p)\epsilon}$, then  $f_\commset(\epsilon^*;p) \!=\! \tilde{\gamma}_g^2(\meanvec;\commset)$.
\end{restatable}
An important conclusion is that under the assumptions of the lemma, $\tilde{\gamma}_g^2(\meanvec;\commset)>0$, since $f_\commset(\epsilon^*;p)=c^TB_{KL}^{-1}(p)c>0$. A more general result naturally arises from the proof of Lemma \ref{lemma:boundExplicitCalc}: for any for any $\meanvec\in\brs*{0,1}^\Nbatch$ and any $\commset\subset\brs*{\Nbatch}$, it holds that $\tilde{\gamma}_g^2(\meanvec;\commset)\ge0$, as expected from a smoothness parameter. We refer the readers to the proof of the lemma for additional details. Substituting back into \eqref{eq:lower_bound_dependent_form} and recalling that $\commset^*$ was chosen as the maximizer in $DB^*_r(\dr{};\meanvec)$ we get
{\small
\begin{align*}
    \liminf_{T\to\infty} \frac{\E_\unu\brs*{R(T)}}{\ln T}
    & \ge  \frac{(\Narms-2\Nbatch)\tilde{\gamma}_g^2(\meanvec;\commset^*)(p)}{8\Ndis^*\dr{\epsilon}} 
    = DB^*_r(\dr{\epsilon};\meanvec) .
\end{align*}
}
We finally return to our assumption that $p_1<\dots<p_\Ndis$. If there are equal values in $p$, i.e., $p_i=p_i=\dots=p_{i+n-1}$, we modify both distributions in Table \ref{table:lower_bound_distribution} such that the observations of these arms are identical, namely $\obs[i]{t}=\obs[i+1]{t}=\dots=\obs[i+n-1]{t}$. Then, we set $\epsilon^*_{i+1}=\dots=\epsilon^*_{i+n-1}=0$. Notice that this modification does not change the KL divergence, nor the analysis of the gap, and therefore retains the same results. Similarly, Lemma \ref{lemma:boundExplicitCalc} still holds by defining the function $f_\commset$ over the sub-vector of $\epsilon$ with coordinates such that $p_i<p_{i+1}$. Alternatively, note that the existing analysis naturally sets $\epsilon_i=0$ if $p_i=p_{i-1}$ (Lemma \ref{lemma:boundExplicitCalc}), so it is not surprising that this modification does not change the results. 
We avoided writing the full analysis since it requires indexing the sub-vector of $p$ with strictly increasing values, which would make the notations much more involved. 

To conclude the proof, we remark that if $\dr{0}>0$ is the gap for some $\epsilon_0>0$ such that Lemmas \ref{lemma:kl-bound} and \ref{lemma:gap-bound} hold, by tuning $\epsilon_0$ we can achieve any gap $\dr{}\le\dr{0}$, and thus the previous result holds for any small enough gap $\dr{}$. 
\end{proof}

\section{Problem-Independent Lower Bounds}
In this section, we prove a problem-independent regret lower bound. Specifically, we prove that for any fixed strategy and any large enough horizon $T$, there exists a CMAB instance such that the regret is lower bounded by a gap-independent $\sqrt{T}$ term. We remark that in contrast to problem-dependent bounds, in which the instance is fixed for \emph{all} strategies and time horizons, the instance for problem-independent bounds is designed as the `worst-case' problem for a \emph{specific} strategy and time horizon. Similarly to the previous section, we start by proving a general lower bound for $\commset$-disjoint CMAB problems and then apply it with a specific distribution to derive the desired bound.
\begin{restatable}{lemma-rst}{generalIndependentBound}
\label{lemma:generalIndependentBound}
Let $\rewardNoArgs$ be a smooth index invariant reward function, and let $\commset\subset\brs*{\Narms}$ such that $\abs*{\commset}\le\Nbatch$ and $\Narms>\Nbatch$. 
Also, let $\meanvec,\meanvec^*\in\R^\Nbatch$ such that $\meanvec_\commset=\meanvec^*_\commset$ and $\dr{}=\rewardvec{\meanvec^*}-\rewardvec{\meanvec}>0$. Finally,  let $\nu,\nu^*$  be two distributions with expectations $\meanvec,\meanvec^*\in\R^\Nbatch$, such that arms in $\commset$ are mutually independent of arms outside $\commset$ and both distributions are identical for arms in $\commset$. 
Then, for any horizon $T$ and any strategy $\psi$, there exists an $\commset$-disjoint CMAB problem with arm distribution $\unu'$ such that $\nu_{\action^*}=\meanvec^*$ and its regret under strategy $\psi$ is bounded by
{\small
\begin{align*}
    \E_{\unu'}\brs*{R(T)} \!\ge\! T\dr{}\br*{\!1-\frac{\Ndis}{\Narms-\Nbatch} -\sqrt{\frac{1}{2}\frac{T\Ndis}{\Narms-\Nbatch}\kl(\nu,\nu^*)}}\!.
\end{align*}
}
\end{restatable}
The proof is a variant of Theorem 6 of \citet{garivier2018explore} and can be found in Appendix \ref{appendix:general-bounds}. With this lemma at hand, and similarly to the problem-dependent bound of Theorem \ref{theorem:dependent_lower_bound}, we can also derive a problem-independent lower bound:
\begin{restatable}{theorem-rst}{independentLowerBound}
\label{theorem:independent_lower_bound}
Let $\rewardNoArgs$ be a smooth index invariant reward function and assume that $\Narms\ge3\Nbatch$. 
Then, for any $\meanvec\!\in\!\brs*{0,1}^\Nbatch$, any $T\!\ge\! T_0$ and for any strategy $\psi$, there exists an $\commset$-disjoint CMAB problem $\unu'$ with $\E\brs*{\nu'_{\action^*}}=\meanvec$ such that its regret under strategy $\psi$ is bounded by 
{\small
\begin{align*}
    \E_{\unu'}\brs*{R(T)} \ge \max_{\commset}\frac{\tilde{\gamma}_g(\meanvec;\commset)}{32}\sqrt{\frac{T(\Narms-\Nbatch)}{\Ndis}} \triangleq IB^*_r(\meanvec) \enspace .
\end{align*} }
\end{restatable}
Similarly to Corollary \ref{corollary: sum index invariant dependent}, the result can also be easily extended to sums as follows:
\begin{corollary}
\label{corollary: sum index invariant independent}
Under the notations of Corollary \ref{corollary: sum index invariant dependent}, if $\Narms\ge3\Nbatch$, then for any $\meanvec\!\in\!\brs*{0,1}^\Nbatch$, any $T\!\ge\! T_0$ and for any strategy $\psi$, there exists CMAB problem $\unu'$ such that the optimal action has means $\meanvec_i=\meanvec$ for all $i\in\brs*{\Nitems}$ whose regret under strategy $\psi$ is bounded by 
{\small
\begin{align*}
    \E_{\unu'}\brs*{R(T)} \ge \Nitems\cdot IB^*_r(\meanvec) \enspace .
\end{align*} }
\end{corollary}
We defer the proof of the theorem to Appendix \ref{appendix:problem independent proof}, and the corollary can be proven similarly to Corollary \ref{corollary: sum index invariant dependent}.  
The same discussion from the previous section about the tightness of the bound still holds. We start by noting that $IB^*_r(\meanvec)$ reproduces the existing lower bounds both for the linear reward function \citep{kveton2015tight} and the probabilistic maximum coverage problem \citep{merlisM19}. Also, for monotone functions, we can bound $\tilde{\gamma}_g^2(\meanvec;\emptyset)\ge\gamma_g^2$ and match the problem-independent upper bound of \citep{merlisM19} up to a $\sqrt{\Nbatch}$ factor. This factor will be improved to a logarithmic factor in the following section.


\section{Relations Between Smoothness Measures}
\label{section:relations}
To this point, we derived lower bounds that depend on the modified Gini-smoothness $\tilde{\gamma}_g(\meanvec;\commset)$. In this section, we show that at a cost of logarithmic factors, we can relate these bounds to the L1 Gini-smoothness. Moreover, for monotone reward functions, we also prove lower bounds that depend on the L2 Gini-smoothness, and thus match the upper bounds of \citep{merlisM19} up to log-factors. We formally state the results in the following propositions:
\begin{restatable}{proposition-rst}{smoothnessRelationModified}\label{prop:smoothness-relation-modified-L1}
Let $\rewardNoArgs$ be a smooth index invariant reward function and denote $p=p^{\meanvec,\commset}$ for some $\meanvec\in\R^\Nbatch$ and $\commset\subset\brs*{\Nbatch}$. Then,
{\small
\begin{align}
\label{eq:smoothness-relation-modified-L1}
     \tilde{\gamma}_g^2(\meanvec;\commset)\ge\frac{\gamma_{g,1}^2(\meanvec;\commset)}{3 + \ln\frac{1}{p_1} + \ln\frac{1}{1-p_\Ndis}}\enspace ,
\end{align}
}
where the r.h.s. is defined as zero if $p_1=0$ or $p_\Ndis=1$.
\end{restatable}
\begin{restatable}{proposition-rst}{smoothnessRelationEuclid}\label{prop:smoothness-relation-euclid}
Let $\rewardNoArgs$ be a monotone smooth index invariant reward function. Then, for any $\meanvec\in\R^\Nbatch$, it holds that 
{\small
\begin{align}
\label{eq:smoothness-relation-L1-L2}
    \max_{\commset}&\frac{\gamma_{g,1}^2(\meanvec;\commset)}{\Ndis} \ge \frac{\gamma_{g,2}^2(\meanvec;\emptyset)}{1+\ln\Nbatch}\enspace.
\end{align}
}
Furthermore, if $\meanval_{\min}\!=\!\min_{i:\meanval_i>0} \meanval_i$ and $\meanval_{\max}\!=\!\max_{i:\meanval_i<1} \meanval_i$, then
{\small
\begin{align*}
    \max_{\commset}&\frac{\tilde{\gamma}_g^2(\meanvec;\commset)}{\Ndis} 
    \!\ge\! \frac{\gamma_{g,2}^2(\meanvec;\emptyset)}{\br*{1+\ln\Nbatch}\br*{3 + \ln\frac{1}{\meanval_{\min}} + \ln\frac{1}{1-\meanval_{\max}}}}.
\end{align*}
}
\end{restatable}
We emphasize that Proposition \ref{prop:smoothness-relation-modified-L1} \emph{does not} require the monotonicity assumption. However, when the components of the gradient can be either positive or negative, the bound might equal zero even when the gradient is large. When the function is monotone, Proposition \ref{prop:smoothness-relation-euclid} greatly improves the na\"ive choice of $\commset=\emptyset$ in the bounds of Theorems \ref{theorem:dependent_lower_bound} and \ref{theorem:independent_lower_bound}, by a factor of $\Nbatch$, with only a logarithmic cost. Specifically, by maximizing over $\meanvec$ and applying Proposition \ref{prop:smoothness-relation-euclid}, we get a problem-dependent lower bound of  $\tilde\Omega\br*{\Narms\gamma_g^2\ln T/\dr{}}$ and a problem-independent bound of $\tilde\Omega\br*{\gamma_g\sqrt{\Narms T}}$, both match the upper bounds of \citep{merlisM19} up to log-factors. Thus, for monotone smooth functions, we conclude that the L2 Gini-smoothness parameter characterizes both the upper and lower bounds. 

Of the log-factors in the propositions, the more interesting one is that of $\br*{\ln\frac{1}{\meanval_{\min}} + \ln\frac{1}{1-\meanval_{\max}}}$. In cases where the mean values $\meanvec$  are exponentially close to zero or one, this factor can be of order $1/\Nbatch$. We suspect that this is due to a proof artefact, but leave the investigation of this factor to future work. Nonetheless, for any practical example, this term is at most of order $\ln\Nbatch$, which leaves the bound tight up to log-factors in the problem size. 
Another question that arises is whether this factor is the result of a loose analysis in Proposition \ref{prop:smoothness-relation-modified-L1}, and a tighter analysis might yield a better factor (e.g., $\ln\Nbatch$). Sadly, the answer is negative. In Appendix \ref{appendix:tightness modified L1}, we build an example where $\meanval_i$ are exponentially small and Inequality \eqref{eq:smoothness-relation-modified-L1} \emph{does not hold} without a factor of order $1/\Nbatch$. Therefore, to remove this factor, $DB^*_r(\dr{};\meanvec)$ and $IB^*_r(\meanvec)$ also need to be improved.

Due to space limitations, we defer the full proofs to Appendix \ref{appendix:relations} and only provide a proof sketch for Proposition \ref{prop:smoothness-relation-euclid}:
\begin{proofsketch} 
Notice that $\gamma_{g,1}(\meanvec;\!\commset)$ and $\gamma_{g,2}(\meanvec;\!\commset)$ are closely related to the L1 and L2 norms of a vector whose components are $\sqrt{\meanval_i(i-\meanval_i)}\nabla_i\rewardvec{\meanvec}$. Specifically, both are norms of a vector with a subset of these components. We utilize this connection and derive, to the best of our knowledge, a new relation between the norms:
\begin{restatable}{lemma-rst}{normsRelation} \label{lemma:normsRelation}
Let $A$ be a nonempty subset of indices $A\subset\brs*{n}$. For any vector $x\in\R^n$, it holds that
{\small
\begin{align*}
    \max_{A\ne\emptyset} \frac{1}{\abs*{A}} \norm*{x_A}_1^2 \ge \frac{1}{1+\ln n} \norm{x}_2^2\enspace.
\end{align*}
}
\end{restatable}
This lemma gives a much stronger relation than the standard L1-L2 inequality and is of independent interest. Specifically, the relation between the norms is logarithmic in the vector dimension, instead of the standard relation $\norm*{x}_1\ge \norm*{x}_2$, that would have resulted in a linear term. This is due to the ability to choose the best subset of the vector on the left-hand side. We remark that this inequality is tight, as demonstrated in Appendix \ref{appendix:tightness L1 L2}. By applying Lemma \ref{lemma:normsRelation}, we get $\max_\commset \frac{\gamma_{g,1}^2(\meanvec;\commset)}{\Ndis} \ge \frac{\gamma_{g,2}^2(\meanvec,\emptyset)}{1+\ln\Nbatch}$, and substituting in Proposition \ref{prop:smoothness-relation-modified-L1} concludes the proof.
\end{proofsketch}

\section{Summary and Future Work}
In this work, we presented the first lower bounds for the CMAB problem that hold for general reward functions, under very mild assumptions. Specifically, we proved both problem-dependent and problem-independent lower bounds, which depend on the modified Gini-smoothness $\tilde{\gamma}_g(\meanvec;\commset)$ and reproduce the existing bounds for specific instances. When the reward function is also monotone, we showed that the upper bounds of \citep{merlisM19}, which depend on the L2 Gini smoothness of the reward function, are tight up to logarithmic factors. There are a few directions for extending our results that we leave for future work:

\textbf{Gini-smoothness and non-monotone reward functions}:  One question that naturally arises is whether the L2 Gini-smoothness $\gamma_g$ also characterizes the lower bound for non-monotone reward function. If such dependence truly holds, a possible way to derive these bounds is by improving $\tilde{\gamma}_g(\meanvec;\commset)$ such that it depends on the absolute values of the gradient components. However, such modification of the analysis is highly nontrivial, and we leave it for future work.

\textbf{Lower bounds for arbitrary action sets}:  To derive the lower bounds, we carefully designed the action set of the problem, such that no information is gained on one action by sampling a different one. In practice, different actions might have overlapping arms, which can be sometimes used to achieve better performance. For example, in the linear reward function and when the action set contains all possible subsets of fixed size, superior regret bounds can be attained \citep{komiyama2015optimal}. To the best of our knowledge, the only similar result is by \citep{combes2015combinatorial} for the specific case of linear rewards and independent arms. They show that the best achievable performance strongly depends on the structure of the action set, and it is interesting to derive such lower bounds for general reward functions and arm distributions. 

\textbf{Distribution-dependent lower bounds}: To derive lower bounds that depend on the Gini-smoothness, we designed a family of arm distributions, all with binary support. Similarly, to derive the upper bounds, \citet{merlisM19} bounded the variance of the arms by the variance of Bernoulli arms. We can, therefore, conclude that Bernoulli distribution is the `worst-case' distribution, under which both the upper and the lower bounds depend on the Gini-smoothness. A possible future direction is analyzing both bounds under general distributions and deriving new smoothness criteria that depend on the specific arm distribution, rather than the worst-case distribution.

\textbf{Other variants and performance measures}: In this work, we focused on regret lower bounds for the CMAB problem with semi-bandit feedback. Other interesting problems include the case of full-bandit feedback \citep{gopalan2015thompson,rejwan2020top}, where there is no feedback on specific arms, but rather on the reward of the action, or using sample complexity as the performance measure \citep{chen2017nearly,mannor2004sample,kaufmann2016complexity}. Both variants have mainly been studied for the linear reward functions, and extending the existing upper and lower bounds for general reward functions can be beneficial for many practical settings.

\acks{This work was partially funded by the Israel Science Foundation under ISF grant number 1380/16.}

\bibliography{references}

\appendix

\newpage


\section{General Lower bound for \commset-disjoint CMAB Problems} \label{appendix:general-bounds}
\generalDependentBound*

\begin{proof}
For any arm distribution $\unu$ and any suboptimal action $\action$, consider a modified $\commset$-disjoint problem with arm distribution $\unu'$, where $\nu_A'=\nu_A$ for all $A\ne\action$, and thus $\reward{A}{\meanvec'}=\reward{A}{\meanvec}$, and $\nu'_\action$ is modified such that $\reward{\action}{\meanvec'}>\reward{\action^*}{\meanvec}$ and $\nu'_\action$ has the same support as $\nu_{\action^*}$. Specifically, $\nu'_\action$ can be constructed by initializing $\nu'_\action\leftarrow\nu_{\action^*}$, and then modifying arm $i$ in the direction of the gradient. Due to the assumption, the arm is not deterministic, and there exist arm distributions with the same support and expectations $\meanval_i'>\meanval_i$ and $\meanval_i'<\meanval_i$. Thus, modifying it in the direction of the gradient is valid, and since the gradient is not zero, the reward will increase. Furthermore, as the modification is done for an arm $i\notin\commset$, it does not affect any other action, and the new instance is still $\commset$-disjoint.

We now apply Lemma \ref{lemma:kl count bound} with $Z\!=\!N_{\psi,\action}(T)/T$, while noting that $\kl(\nu_A,\nu_A')\!=\!0$ for all $A\!\ne\!\action$:
\begin{align} 
    \label{eq:probDependMiddle}
    \E_\unu\brs*{N_{\psi,\action}(T)} \kl(\nu_\action,\nu_\action') 
    &\ge \klBin\br*{\frac{\E_{\unu}\brs*{N_{\psi,\action}(T)}}{T},\frac{\E_{\unu'}\brs*{N_{\psi,\action}(T)}}{T}}\nonumber \\
    &\ge \br*{1 - \frac{\E_{\unu}\brs*{N_{\psi,\action}(T)}}{T}}\ln\br*{\frac{T}{T - \E_{\unu'}\brs*{N_{\psi,\action}(T)}}} - \ln2 \enspace .
\end{align}
where the second inequality uses the following bound which holds for any $p,q\in[0,1]$
\begin{align*}
    \klBin(p,q) 
    = \underbrace{p\ln\frac{1}{q}}_{\ge0} + (1-p)\ln\frac{1}{1-q} + \underbrace{p\ln p + (1-p)\ln(1-p)}_{\ge-\ln2}
    \ge (1-p)\ln\frac{1}{1-q}-\ln2 \enspace .
\end{align*}

Next, as $\psi$ is consistent and all actions $A\ne\action$ are strictly suboptimal for bandit problem $\unu'$, we get that for any $0<\alpha\le1$,
\begin{align*}
    0 \le T - \E_{\unu'}\brs*{N_{\psi,\action}(T)} = \sum_{A\ne\action} \E_{\unu'}\brs*{N_{\psi,A}} = o(T^\alpha) \enspace .
\end{align*}
In particular, for sufficiently large $T$, it holds that $T - \frac{\E_{\unu'}N_{\psi,\action}(T)}{T}\le T^\alpha$, and therefore, for any $0<\alpha\le1$,
\begin{align*}
    \liminf_{T\to\infty}\frac{1}{\ln T}\ln\br*{\frac{T}{T - \E_{\unu'}N_{\psi,\action}(T)}}
    \ge \liminf_{T\to\infty}\frac{1}{\ln T}\ln\br*{\frac{T}{T^\alpha}} = 1-\alpha \enspace.
\end{align*}
Furthermore, since $\psi$ is consistent and $\action$ is suboptimal in $\unu$, $\E_\unu\brs*{N_{\psi,\action}(T)}/T\to0$, and substituting both inequalities into \eqref{eq:probDependMiddle} yields
\begin{align*}
    \liminf_{T\to\infty}\frac{\E_\unu\brs*{N_{\psi,\action}(T)}}{\ln T} \ge \frac{1}{\kl(\nu_\action,\nu_{\action}')} \enspace.
\end{align*}

Next, note that if $\kl(\nu_\action,\nu_{\action^*})=\infty$, the result of the lemma trivially holds. Otherwise, recall that there exist $i\notin\commset$ in $\action^*$ with $\Pr\brc*{\obs[i]{t}=\meanval_i}<1$ and $\nabla_i\reward{\action^*}{\meanvec}\ne0$, and without loss of generality assume that $\nabla_i\reward{\action^*}{\meanvec}>0$. By modifying only this component, we can build a sequence of distributions ${\nu'}^n$ such that the following holds:
\begin{enumerate}
    \item For all $n$, ${\meanval'_i}^n>\meanval_i^*$ and ${\meanval'_j}^n=\meanval_j^*$ for all $j\ne i \in \action^*$.
    \item $\lim_{n\to\infty} {\meanval'_i} = \meanval_i^*$.
    \item $\lim_{n\to\infty} \kl\br*{\nu_\action,{\nu'_S}^n}=\kl\br*{\nu_\action,\nu_{\action^*}}$.
\end{enumerate}
By the index invariance assumption, and since $\nabla_i\reward{\action^*}{\meanvec}>0$, for large enough $n$ it holds that $\reward{\action}{{\meanvec'}^n}>\reward{\action^*}{\meanvec}$, and by taking the infimum over all of these distributions we get for all suboptimal actions
\begin{align*}
    \liminf_{T\to\infty}\frac{\E_\unu\brs*{N_{\psi,\action}(T)}}{\ln T} \ge \frac{1}{\kl(\nu_\action,\nu_{\action^*})} \enspace .
\end{align*}
Similarly, if $\nabla_i\reward{\action^*}{\meanvec}<0$, we require that ${\meanval'_i}^n<\meanval_i^*$, and the same arguments hold. To derive the second part of the lemma, notice that $R(T) = \sum_{\action\in\actionset} \dr{\action}\E_\unu\brs*{N_{\psi,\action}(T)}$, and substituting the previous bound for all of the suboptimal actions yields the desired result.
\end{proof}

\clearpage
\generalIndependentBound*
\begin{proof}
Define an $\commset$-disjoint CMAB problem with the maximal action set such that each action contains $\Nbatch$ arms, i.e., $\abs{\actionset} = \floor*{\frac{\Narms-\abs{\commset}}{\Nbatch-\abs{\commset}}} \ge \frac{\Narms-\Nbatch}{\Ndis}$, where all actions are distributed according to $\nu_\action=\nu$. We denote this problem by $\unu$. Due to the pigeonhole principle, there exists an action $\action^*$ such that under strategy $\psi$, $\E_\unu\brs*{N_{\psi,\action^*}(T)} \le \frac{T}{\abs{\actionset}}$. Next, define a modified bandit problem $\unu'$ such that $\nu'_\action=\nu$ for all $\action\ne\action^*$ and $\nu_{\action^*}=\nu^*$. Thus,
\begin{align}
\label{eq:problem_independent_regret_decomp}
    \E_{\unu'}\brs*{R(T)} = \sum_{\action\ne\action^*} \dr{}\E_{\unu'}\brs*{N_{\psi,\action}(T)} = T\dr{}\br*{1 - \frac{\E_{\unu'}\brs*{N_{\psi,\action^*}(T)}}{T} } \enspace.
\end{align}

As we only changed $\action^*$, for all $\action\ne\action^*$, $\kl(\nu_\action,\nu'_\action)=0$. Thus, by applying Lemma \ref{lemma:kl count bound} on the random variable $Z=N_{\psi,\action^*}(T)/T$ and then using Pinsker's Inequality, we get 
\begin{align*}
    \E_{\unu}\brs*{N_{\psi,\action^*}(T)}\kl(\nu,\nu^*) 
    &\ge \klBin\br*{\frac{\E_{\unu}\brs*{N_{\psi,\action^*}(T)}}{T}, \frac{\E_{\unu'}\brs*{N_{\psi,\action^*}(T)}}{T} }\\
    &\ge 2\br*{\frac{\E_{\unu}\brs*{N_{\psi,\action^*}(T)}}{T} - \frac{\E_{\unu'}\brs*{N_{\psi,\action^*}(T)}}{T} }^2 \enspace.
\end{align*}
It can be easily verified that when $\E_{\unu'}\brs*{N_{\psi,\action^*}(T)}/T$ is either smaller or larger than $\E_{\unu}\brs*{N_{\psi,\action^*}(T)}/T$, it holds that
\begin{align*}
    \frac{\E_{\unu'}\brs*{N_{\psi,\action^*}(T)}}{T} \le \frac{\E_{\unu}\brs*{N_{\psi,\action^*}(T)}}{T} + \sqrt{ \frac{1}{2} \E_{\unu}\brs*{N_{\psi,\action^*}(T)}\kl(\nu,\nu^*) }\enspace ,
\end{align*}
Substituting $\E_\unu\brs*{N_{\psi,\action^*}(T)} \le \frac{T}{\abs{\actionset}}$ and $\abs{\actionset}\ge \frac{\Narms-\Nbatch}{\Ndis}$ into the regret bound \eqref{eq:problem_independent_regret_decomp} leads to the desired result.
\end{proof}

\clearpage


\section{Technical Lemmas} \label{appendix:technical}
\klBound*
\begin{proof}
Assume that $\norm{\epsilon}_\infty\le b_{\epsilon,0}$, such that $\nu$ is a valid probability distribution. Since all arms $i\in\commset^*$ are mutually independent of all arms $i\notin\commset^*$, we can write  
\begin{align*}
    \kl(\nu,\nu^*) =\underbrace{\kl(\nu_{\commset^*},\nu^*_{\commset^*})}_{=0} + \kl(\nu_{{\commset^*}^c},\nu^*_{{\commset^*}^c}) = \kl(\nu_{{\commset^*}^c},\nu^*_{{\commset^*}^c})
\end{align*}
Next, the KL-divergence between the distributions in Table \ref{table:lower_bound_distribution} can be bounded by:
{\small
\begin{align*}
    \kl(\nu,\nu^*) 
    &= p_1\ln\frac{p_1}{p_1-\epsilon_1} + \sum_{j=2}^{\Ndis} \br*{p_j-p_{j-1}}\ln\frac{p_j-p_{j-1}}{p_j-p_{j-1}-\epsilon_j} 
    + (1-p_\Ndis)\ln\frac{1-p_\Ndis}{1-p_\Ndis +\sum_{j=1}^\Ndis \epsilon_j} \\
    & \stackrel{(*)}{\le} \frac{p_1\epsilon_1}{p_1-\epsilon_1} + \sum_{j=2}^{\Ndis} \frac{\br*{p_j-p_{j-1}}\epsilon_j}{p_j-p_{j-1}-\epsilon_j} - \frac{(1-p_\Ndis)\sum_{j=1}^\Ndis \epsilon_j}{1-p_\Ndis+\sum_{j=1}^\Ndis \epsilon_j} \\
    & = \epsilon_1 + \frac{\epsilon_1^2}{p_1-\epsilon_1} + \sum_{j=2}^{\Ndis} \br*{\epsilon_j + \frac{\epsilon_j^2}{p_j-p_{j-1}-\epsilon_j}} -\sum_{j=1}^\Ndis \epsilon_j + \frac{\br*{\sum_{j=1}^\Ndis \epsilon_j}^2}{1-p_\Ndis+\sum_{j=1}^\Ndis \epsilon_j} \\
    & = \frac{\epsilon_1^2}{p_1-\epsilon_1} + \sum_{j=2}^{\Ndis} \frac{\epsilon_j^2}{p_j-p_{j-1}-\epsilon_j} + \frac{\br*{\sum_{j=1}^\Ndis \epsilon_j}^2}{1-p_\Ndis+\sum_{j=1}^\Ndis \epsilon_j}
\end{align*}
}
where $(*)$ is due to the inequality $\ln x \le x-1$. Now let $$b_{\epsilon,1}=\frac{1}{2}\min\brc*{p_1,\min_i(p_i-p_{i-1}),\frac{1}{\Ndis}(1-p_\Ndis),b_{\epsilon,0}}\enspace, $$
and for any $\epsilon\in\R^\Ndis$ such that $\norm*{\epsilon}_\infty\le b_{\epsilon,1}$, it holds that
\begin{align*}
    \kl(\unu,\unu^*) 
    \le 2\br*{ \frac{\epsilon_1^2}{p_1} + \sum_{j=2}^{\Ndis} \frac{\epsilon_j^2}{p_j-p_{j-1}} + \frac{\br*{\sum_{j=1}^\Ndis \epsilon_j}^2}{1-p_\Ndis}}\enspace.
\end{align*}
Writing the bound in a matrix formulation yields the desired result and concludes the proof.
\end{proof}

\clearpage

\gapBound*
\begin{proof}
Define $\ab_j = \sum_{i=1}^j \ub_i$ for all $j\in\brs*{\Ndis}$ and  $\ab_j=0$ for all $j>\Ndis$, which also results with
\begin{align*}
    \ab^T\nabla\rewardvec{p} 
    = \sum_{j=1}^\Ndis\sum_{i=1}^j \ub_i\nabla_j\rewardvec{p} 
    = \sum_{j=1}^\Ndis \ub_j \sum_{i=j}^\Ndis \nabla_i \rewardvec{p} 
    = \sum_{j=1}^\Ndis \ub_j c_j
    = c^T\ub>0 \enspace.
\end{align*}
Next, by the definition of the gradient, 
\begin{align*}
    \lim_{\epsilon_0\to0}\frac{\rewardvec{p} - \rewardvec{p-\epsilon_0\ab}}{\epsilon_0} = \ab^T\nabla\rewardvec{p} \enspace.
\end{align*}
Specifically, for $\delta=c^T\ub/2>0$, there exists $b_{\epsilon,2}>0$, such that for all $0<\epsilon_0\le b_{\epsilon,2}$, 
\begin{align*}
    \rewardvec{p} - \rewardvec{p-\epsilon_0\ab} \ge \epsilon_0\br*{\ab^T\nabla\rewardvec{p} - \delta} = \frac{1}{2}\epsilon_0 c^T\ub>0 \enspace.
\end{align*}

We now show that the l.h.s of the inequality is equal to $\dr{\epsilon}$. First, Under the index invariance assumption, it holds that $\rewardvec{p}=\rewardvec{\meanvec^{\nu^*}}$. Also, note that when comparing to $\nu^*$, the mean value of base arm $i$ in distribution $\nu$ decreases by $\sum_{j=1}^i \epsilon_j = \epsilon_0\ab_i$, and therefore $\rewardvec{p-\epsilon_0\ab} = \rewardvec{\meanvec^\nu}$. Finally, recall that $\dr{\epsilon}=\rewardvec{\meanvec^{\nu^*}} - \rewardvec{\meanvec^\nu}$, and substituting into the last inequality yields the desired result.
\end{proof}


\boundExplicitCalc*
\begin{proof}
By applying the Sherman-Morrison Formula \citep{hager1989updating} on $B_{KL}(p)$, we get
\begin{align*}
    B^{-1}_{KL}(p) = \br*{D(p)+\frac{1}{1-p_\Ndis}\onevec\onevec^T}^{-1} 
    = D^{-1}(p) - \frac{1}{1-p_\Ndis}\frac{D^{-1}(p)\onevec\onevec^TD^{-1}(p)}{1+\frac{1}{1-p_\Ndis}\onevec^TD^{-1}(p)\onevec}\enspace .
\end{align*}
$D^{-1}(p)$ is a diagonal matrix whose elements are $D^{-1}_{ii}(p) = \frac{1}{D_{ii}(p)} = p_i-p_{i-1}$. The elements of the matrix $D^{-1}(p)\onevec\onevec^TD^{-1}(p)$ are 
\begin{align*}
    \br*{D^{-1}(p)\onevec\onevec^TD^{-1}(p)}_{ij} = D^{-1}(p)_{ii}D^{-1}(p)_{jj}  = \br*{p_i-p_{i-1}}\br*{p_j-p_{j-1}}
\end{align*}
and denominator can be written as 
\begin{align*}
    1+\frac{1}{1-p_\Ndis}\onevec^TD^{-1}(p)\onevec
    &= 1 + \frac{1}{1-p_\Ndis}\sum_{i=1}^{\Ndis}\frac{1}{D_{ii}(p)} \\
    &= 1 + \frac{1}{1-p_\Ndis}\sum_{i=1}^{\Ndis}\br*{p_i-p_{i-1}} \\
    &= 1 + \frac{p_\Ndis}{1-p_\Ndis}\\
    &=\frac{1}{1-p_\Ndis}
    \enspace .
\end{align*}
where we used the fact that $p_0=0$. Therefore, the elements of $B^{-1}_{KL}(p)c$ are equal to 
\begin{align*}
    \br*{B^{-1}_{KL}(p)c}_i 
    &= \br*{D^{-1}(p)c}_i - \frac{1}{1-p_\Ndis}\br*{\frac{D^{-1}(p)\onevec\onevec^TD^{-1}(p)c}{1+\frac{1}{1-p_\Ndis}\onevec^TD^{-1}(p)\onevec}}_i \\
    &= \br*{p_i-p_{i-1}}c_i - \sum_{j=1}^\Ndis\br*{p_i-p_{i-1}}\br*{p_j-p_{j-1}}c_j\\ 
    &= \br*{p_i-p_{i-1}}\br*{c_i - \sum_{j=1}^\Ndis\br*{p_j-p_{j-1}}c_j}
    \enspace .
\end{align*}
Recalling that $\epsilon^* = \epsilon_0B^{-1}_{KL}(p)c$ concludes the first part of the lemma. For the second part of the lemma, we directly substitute $\epsilon^*$:
\begin{align}
    f_\commset(\epsilon^*;p) \nonumber
   &  = c^TB^{-1}_{KL}(p)c\\
   & = \sum_{i=1}^{\Ndis}\br*{p_i-p_{i-1}}c_i\br*{c_i - \sum_{j=1}^\Ndis\br*{p_j-p_{j-1}}c_j} \nonumber\\
    & = \sum_{i=1}^{\Ndis}\br*{p_i-p_{i-1}}c_i^2 - \br*{\sum_{i=1}^{\Ndis}c_i\br*{p_i-p_{i-1}}}^2 \label{eq:smooth_positive}\\
    & \stackrel{(*)}{=} \sum_{i=1}^{\Ndis}p_i \br*{c_i^2-c_{i+1}^2} - \br*{\sum_{i=1}^{\Ndis}p_i \br*{c_i-c_{i+1}}}^2 \nonumber\enspace ,
\end{align}
where $(*)$ is due to the following identities and under the notations $p_0=c_{\Ndis+1}=0$:
\begin{align*}
    &\sum_{i=1}^\Ndis \br*{p_i-p_{i-1}}c_i = \sum_{i=1}^\Ndis p_i \br*{c_i-c_{i+1}} \\
    &\sum_{i=1}^\Ndis \br*{p_i-p_{i-1}}c_i^2 = \sum_{i=1}^\Ndis p_i \br*{c_i^2-c_{i+1}^2}  \enspace.
\end{align*} 
This term can be further simplified by substituting $c_i = \sum_{j=i}^\Ndis \nabla_j \rewardvec{p}$ as follows:

\begin{align*}
   f_\commset(\epsilon^*;p) 
    & = \sum_{i=1}^\Ndis p_i \br*{c_i^2-c_{i+1}^2} - \br*{\sum_{i=1}^\Ndis p_i \br*{c_i-c_{i+1}}}^2 \nonumber\\
    &=\sum_{i=1}^\Ndis p_i \nabla_i\rewardvec{p}^2 + 2\sum_{i=1}^\Ndis p_i \sum_{j=i+1}^\Ndis\nabla_i\rewardvec{p}\nabla_j\rewardvec{p} 
    - \sum_{i=1}^\Ndis\sum_{j=1}^\Ndis p_ip_j \nabla_i\rewardvec{p}  \nabla_j\rewardvec{p} \nonumber\\
    &=\sum_{i=1}^\Ndis p_i \nabla_i\rewardvec{p}^2 + 2\sum_{i=1}^\Ndis p_i \sum_{j=i+1}^\Ndis\nabla_i\rewardvec{p}\nabla_j\rewardvec{p} \\
    &\quad- \sum_{i=1}^\Ndis p_i^2 \nabla_i\rewardvec{p}^2  - 2\sum_{i=1}^\Ndis\sum_{j=i+1}^\Ndis p_ip_j \nabla_i\rewardvec{p}  \nabla_j\rewardvec{p}  \\
    & = \sum_{i=1}^\Ndis p_i(1-p_i)\nabla_i\rewardvec{p}^2 + 2\sum_{i=1}^\Ndis \sum_{j=i+1}^\Ndis p_i(1-p_j)\nabla_i\rewardvec{p}\nabla_j\rewardvec{p} \\
    &= \tilde{\gamma}_g^2(\meanvec;\commset)\enspace . 
\end{align*}

We remark that by Equation \eqref{eq:smooth_positive} and under the notations $p_0=0$, we observe that
\begin{align*}
    \tilde{\gamma}_g^2(\meanvec;\commset) = \sum_{i=1}^{\Ndis}\br*{p_i-p_{i-1}}c_i^2 - \br*{\sum_{i=1}^{\Ndis}c_i\br*{p_i-p_{i-1}}}^2 \enspace.
\end{align*}
Denoting $p_{\Ndis+1}=1$ and $c_{\Ndis+1}=0$, the modified Gini-smoothness can be perceived as the variance of a random variable $X$ such that for all $i\in\brs*{\Ndis+1}$, $\Pr\brc*{X=c_i}=p_i-p_{i-1}$. This also holds when $p_1=0$ and $p_\Ndis=1$, or when $p_i=p_{i+1}$ for some $i\in\brs*{\Ndis-1}$. Thus, we can conclude that for any $\meanvec\in\brs*{0,1}^\Nbatch$ and any $\commset\subset\brs*{\Nbatch}$, it holds that $\tilde{\gamma}_g^2(\meanvec;\commset)\ge0$.
\end{proof}

\clearpage

\section{Proof of Problem Independent Lower Bound}
\label{appendix:problem independent proof}
\independentLowerBound*
\begin{proof}
Without loss of generality, assume that $IB^*_r(\meanvec)>0$, since otherwise the bound trivially holds. As in the proof of Theorem \ref{theorem:dependent_lower_bound}, we denote $p=p^{\meanvec,\commset^*}$, where $\commset^*$ is the maximizer in $IB^*_r(\meanvec)$. We also similarly assume that $0\!<\!p_1\!<\!\dots\!<\!p_\Ndis\!<\!1$ and fix the distributions $\nu$ and $\nu^*$ according to Table \ref{table:lower_bound_distribution}. Finally, we set $\epsilon^*=\epsilon_0B_{KL}^{-1}c$, with $B_{KL}(p)$ and $c$ as in Lemmas \ref{lemma:kl-bound} and \ref{lemma:gap-bound}. Recall that with $\epsilon_0$ small enough, both Lemmas  \ref{lemma:kl-bound} and \ref{lemma:gap-bound} hold. Combining both inequalities, we get
{\small
\begin{align*}
    \kl(\nu,\nu^*)
    = \dr{}^2 \frac{\kl(\nu,\nu^*)}{\dr{}^2} \le \dr{}^2 \frac{2{\epsilon^*}^T B_{KL}(p)\epsilon^*}{\br*{\frac{1}{2}c^T\epsilon^*}^2} 
    = \frac{8\dr{}^2}{f_\commset(\epsilon^*;p)} \enspace,
\end{align*}
}
where $f_\commset(\epsilon^*;p)$ is defined in Lemma \ref{lemma:boundExplicitCalc}. Specifically for our choice of $\epsilon^*$, Lemma \ref{lemma:boundExplicitCalc} also implies that $f_\commset(\epsilon^*;p)=\tilde{\gamma}_g^2(\meanvec;\commset)$. By Lemma \ref{lemma:generalIndependentBound}, there exists an instance of an $I$-disjoint CMAB problem with gap $\dr{}$ and optimal action $\E\brs*{\nu'_{\action^*}}=\meanvec$ such that 
{\small
\begin{align*}
    \E_{\unu'}\brs*{R(T)}
    &\ge T\dr{}\br*{1-\frac{\Ndis}{\Narms-\Nbatch} -\sqrt{\frac{1}{2}\frac{T\Ndis}{\Narms-\Nbatch}\kl(\nu,\nu^*)}} \\
    & \ge T\dr{}\br*{1-\frac{\Ndis}{\Narms-\Nbatch} -2\dr{}\sqrt{\frac{T\Ndis}{(\Narms-\Nbatch)\tilde{\gamma}_g^2(\meanvec;\commset)}}} 
\end{align*}
}
Next, recall that there exists $\dr{0}$ such that we can achieve any gap $0<\dr{}\le\dr{0}$ by tuning $\epsilon_0$. Therefore, for large enough $T$, there exists $\epsilon_0$ such that $\dr{} = \frac{\tilde{\gamma}_g(\meanvec;\commset)}{8}\sqrt{\frac{\Narms-\Nbatch}{T\Ndis}}$, for which we get the bound 
{\small
\begin{align*}
    \E_{\unu'}\brs*{R(T)}
    \ge \frac{\tilde{\gamma}_g(\meanvec;\commset)}{8}\sqrt{\frac{\Narms-\Nbatch}{T\Ndis}}\br*{1-\frac{1}{2} -\frac{1}{4}} 
    = \frac{\tilde{\gamma}_g(\meanvec;\commset)}{32}\sqrt{\frac{T(\Narms-\Nbatch)}{\Ndis}} \enspace.
\end{align*}
}
where in the first inequality we use $\Narms\ge3\Nbatch$. For brevity, we omitted the cases where $p_1\!=\!0$, $p_\Ndis\!=\!1$ or $p_i\!=\!p_{i+1}$ for some $i\!\in\!\brs*{\Ndis}$, and refer the readers to the proof of Theorem \ref{theorem:dependent_lower_bound} for the required modifications.
\end{proof}

\clearpage


 \section{Relations Between the Smoothness Measures}  \label{appendix:relations}

\smoothnessRelationModified*
\begin{proof}
Without loss of generality, assume that $p_1>0$ and $p_\Ndis<1$. Otherwise, the r.h.s. is zero, and since $\tilde{\gamma}_g(\meanvec;\commset)\ge0$, the result trivially holds (further details on the nonnegativity of $\tilde{\gamma}_g(\meanvec;\commset)$ can be found at the end of the proof of Lemma \ref{lemma:boundExplicitCalc}). 
Similarly to Theorem \ref{theorem:dependent_lower_bound}, also assume that $p_1<\dots<p_\Ndis$ and define $B_{KL},c$ and $f_I(\epsilon;p)$ as in Lemmas \ref{lemma:kl-bound}, \ref{lemma:gap-bound} and \ref{lemma:boundExplicitCalc}. As in Theorem \ref{theorem:dependent_lower_bound}, all results can be modified to the case where different arms have the same mean, by defining $\epsilon$ only on indices where $p_i<p_{i+1}$, or equivalently forcing $\epsilon_i=0$ if $p_i=p_{i+1}$. However, we avoid this case for brevity. Finally, we assume that $\Ndis>0$ and $\nabla_i\rewardvec{p}\ne0$ for some $i\in\brs*{\Ndis}$, since otherwise, both sides of the inequality equal zero, and the bound trivially holds.

Next, we denote $p_0=0$ and set $\epsilon_j = \epsilon_0\sqrt{p_j(1-p_j)} - \epsilon_0\sqrt{p_{j-1}(1-p_{j-1})}$ for all $j\in\brs*{\Ndis}$. By construction $\sum_{i=1}^j \epsilon_i = \epsilon_0\sqrt{p_j(1-p_j)}$, and direct substitution yields 
\begin{align}
    \label{eq:gap_specific_epsilon}
    c^T\epsilon 
    = \sum_{j=1}^{\Ndis}\sum_{i=j}^\Ndis\epsilon_j\nabla_i\rewardvec{p} 
    =\sum_{j=1}^{\Ndis}\sum_{i=1}^j\epsilon_i\nabla_j\rewardvec{p} 
    = \epsilon_0\sum_{j=1}^{\Ndis}\sqrt{p_j(1-p_j)}\nabla_j\rewardvec{p}\enspace,
\end{align}
and
\begin{align}
\label{eq:kl_specific_epsilon}
    \epsilon^TB_{KL}(p) \epsilon 
    &= \epsilon_0^2\br*{\sum_{j=1}^{\Ndis} \frac{\br*{\sqrt{p_j(1-p_j)} - \sqrt{p_{j-1}(1-p_{j-1})}}^2}{p_j-p_{j-1}} + p_\Ndis} \nonumber \\
    &= \epsilon_0^2\br*{1-p_1+p_\Ndis+\sum_{j=2}^{\Ndis} \frac{\br*{\sqrt{p_j(1-p_j)} - \sqrt{p_{j-1}(1-p_{j-1})}}^2}{p_j-p_{j-1}}}\enspace.
\end{align}
Before we further bound this term, note that for any $0\le x\le y \le \frac{1}{2}$, it holds that 
\begin{align*}
    \br*{\sqrt{y(1-y)}-\sqrt{x(1-x)}}^2 
    &\le \br*{\sqrt{y(1-y)}-\sqrt{x(1-y)}}^2 
    = (1-y)\br*{\sqrt{y}-\sqrt{x}}^2 \\
    &\le \br*{\sqrt{y}-\sqrt{x}}^2 \enspace,
\end{align*}
where the first inequality uses the fact that $y(1-y)\ge x(1-x)\ge x(1-y)$. From symmetry, it also holds for any $x,y\in[0,\frac{1}{2}]$. If $x\in[0,\frac{1}{2}]$ and $y\in[\frac{1}{2},1]$, we bound
\begin{align*}
    \br*{\sqrt{y(1-y)}-\sqrt{x(1-x)}}^2 
    &= y(1-y)+x(1-x) - 2\sqrt{y(1-y)x(1-x)}\\
    &\le \max\brc*{y(1-y),x(1-x)} - \min\brc*{y(1-y),x(1-x)}  \\
    &= \abs{y(1-y)-x(1-x)} \\ 
    &= (y-x)\abs*{1-x-y} \\
    & \le (y-x)
\end{align*}
If $x,y\in[\frac{1}{2},1]$, from symmetry $\br*{\sqrt{y(1-y)}-\sqrt{x(1-x)}}^2 \le  \br*{\sqrt{1-y}-\sqrt{1-x}}^2$. Recall that $p_j$ are strictly increasing, and denote the last index in which $p_j\le\frac{1}{2}$ by $n$. Applying these inequalities on \eqref{eq:kl_specific_epsilon}, we get 
{\small
\begin{align}
    &\epsilon^TB_{KL}(p) \epsilon \nonumber\\
    &\qquad\le \epsilon_0^2\br*{1-p_1 + p_\Ndis + \sum_{j=2}^{n} \frac{\br*{\sqrt{p_j}-\sqrt{p_{j-1}}}^2 }{p_j-p_{j-1}} + \frac{p_{n+1}-p_n}{p_{n+1}-p_n} + \sum_{j=n+2}^{\Ndis} \frac{\br*{\sqrt{1-p_j}-\sqrt{1-p_{j-1}}}^2 }{p_j-p_{j-1}}} \nonumber\\
    &\qquad\le \epsilon_0^2\br*{3 + \sum_{j=2}^{n} \frac{\br*{\sqrt{p_j}-\sqrt{p_{j-1}}}^2 }{p_j-p_{j-1}} + \sum_{j=n+2}^{\Ndis} \frac{\br*{\sqrt{1-p_j}-\sqrt{1-p_{j-1}}}^2 }{p_j-p_{j-1}}}\enspace. \label{eq:kl_specific_epsilon2}
\end{align}
}
Next, we bound the summands as follows:
\begin{align*}
    \frac{\br*{\sqrt{p_j}-\sqrt{p_{j-1}}}^2 }{p_j-p_{j-1}}
    &= \frac{\br*{\sqrt{p_j}-\sqrt{p_{j-1}}}^2 }{p_j-p_{j-1}} \cdot \frac{\br*{\sqrt{p_j}+\sqrt{p_{j-1}}}^2 }{\br*{\sqrt{p_j}+\sqrt{p_{j-1}}}^2}
    = \frac{\br*{p_j-p_{j-1}}^2 }{(p_j-p_{j-1})\br*{\sqrt{p_j}+\sqrt{p_{j-1}}}^2} \\
    &\le \frac{p_j-p_{j-1} }{p_j}\enspace.
\end{align*}
Similarly, we have
\begin{align*}
    \frac{\br*{\sqrt{1-p_j}-\sqrt{1-p_{j-1}}}^2 }{p_j-p_{j-1}}
    &= \frac{\br*{\sqrt{1-p_j}-\sqrt{1-p_{j-1}}}^2 }{p_j-p_{j-1}} \cdot \frac{\br*{\sqrt{1-p_j}+\sqrt{1-p_{j-1}}}^2 }{\br*{\sqrt{1-p_j}+\sqrt{1-p_{j-1}}}^2} \\
    &= \frac{\br*{p_j-p_{j-1}}^2 }{(p_j-p_{j-1})\br*{\sqrt{1-p_j}+\sqrt{1-p_{j-1}}}^2} \\
    &\le \frac{p_j-p_{j-1} }{1-p_{j-1}}\enspace.
\end{align*}
Substitute both into \eqref{eq:kl_specific_epsilon2} yields:
\begin{align*}
    \epsilon^TB_{KL}(p) \epsilon
    &\le  \epsilon_0^2\br*{ 3 + \sum_{j=2}^{n} \frac{p_j-p_{j-1} }{p_j} + \sum_{j=n+2}^{\Ndis} \frac{p_j-p_{j-1} }{1-p_{j-1}} } \\
    &\stackrel{(*)}{\le} \epsilon_0^2\br*{ 3 + \int_{p_1}^{p_n} \frac{dx}{x}+  \int_{p_{n+1}}^{p_\Ndis} \frac{dx}{1-x}} \\
    &\le \epsilon_0^2\br*{ 3 + \ln\frac{1}{p_1} + \ln\frac{1}{1-p_\Ndis}} \enspace,
\end{align*}
where $(*)$ utilizes the relation between sums and integrals. Combining with \eqref{eq:gap_specific_epsilon} and substituting back into $f_\commset(\epsilon;p)$, we get
\begin{align*}
    f_\commset(\epsilon;p) 
    &\ge \frac{\br*{\epsilon_0\sum_{j=1}^{\Ndis}\sqrt{p_j(1-p_j)}\nabla_j\rewardvec{p}}^2}{\epsilon_0^2\br*{ 3 + \ln\frac{1}{p_1} + \ln\frac{1}{1-p_\Ndis}}} 
    = \frac{\gamma_{g,1}^2(\meanvec;\commset)}{3 + \ln\frac{1}{p_1} + \ln\frac{1}{1-p_\Ndis}}
\end{align*}

The proof is concluded by applying Lemma \ref{lemma:boundExplicitCalc} and recalling that $\epsilon^*$ is the maximizer of $f_\commset(\epsilon;p)$; therefore, for any $\epsilon\ne0$, it holds that $f_\commset(\epsilon;p)\le f_\commset(\epsilon^*;p)\le\tilde{\gamma}_g^2(\meanvec;\commset)$.
\end{proof}

\normsRelation*
\begin{proof}
Without loss of generality, assume that $x_1\ge x_2\ge\dots\ge x_n\ge 0$, as reorganizing and taking absolute values do not affect both sides of the inequality. We bound $\norm{x}_2^2$ as follows:
\begin{align*}
    \norm{x}_2^2 
    &= \sum_{l=1}^n x_l^2 
    \stackrel{(1)}{\le} \sum_{l=1}^n \br*{\frac{1}{l}\sum_{k=1}^l x_k }^2
    = \sum_{l=1}^n \frac{1}{l}\brs*{\frac{1}{l}\br*{\sum_{k=1}^l x_k }^2} \\
    &\stackrel{(2)}{\le} \sum_{l=1}^n \frac{1}{l}\max_{A\ne\emptyset} \frac{1}{\abs*{A}} \norm*{x_A}_1^2
    \stackrel{(3)}{\le} (1+\ln n)\max_{A\ne\emptyset} \frac{1}{\abs*{A}} \norm*{x_A}_1^2
\end{align*}
In $(1)$ we use the fact that $x_i$ are decreasing and non-negative, and thus $x_i\le\frac{1}{i}\sum_{k=1}^i x_k$. For $(2)$, we note that $\frac{1}{l}\br*{\sum_{k=1}^l x_k }^2=\frac{1}{\abs*{A}}\norm*{x_A}_1^2$ for $A=\brc*{1,\dots,l}$, and bound it by the maximum over all possible subsets $A\ne\emptyset$. Finally, $(3)$ uses the well-known property of the harmonic sum $\sum_{l=1}^n \frac{1}{l}\le 1 + \ln n$.
\end{proof}

\clearpage
\smoothnessRelationEuclid*

\begin{proof}
For the first inequality, note that when the function is monotone, all elements in the summation of $\gamma_{g,1}(\meanvec;\commset)$ are nonnegative, and it can be thus conceived as the L1 norm of a vector whose components are $\sqrt{\meanval_i(i-\meanval_i)}\nabla_i\rewardvec{\meanvec}$. Furthermore, $\gamma_{g,2}(\meanvec;\commset)$ can be conceived as the L2 norm of the same vector. Specifically, both are the respective norm of a vector with a subset of these components. Therefore, we can directly relate these two quantities using standard relations between norms. Nonetheless, we are interested in maximizing the lower bound, which includes an additional `penalty' factor on the number of components in the vector $1/\Ndis$. As a result, choosing the largest number of elements in the sub-vector is not always optimal. Specifically, we show in Lemma \ref{lemma:normsRelation} that when optimizing the choice of $\commset$, the penalized $L1$ norm of the sub-vector can be lower bounded by the $L2$ norm of the \emph{full} vector, up to logarithmic factors. 
Applying this lemma on $\gamma_{g,1}(\meanvec;\commset)$ and $\gamma_{g,2}(\meanvec;\emptyset)$ leads to the first inequality of the proposition.

Next, we prove the second inequality. By Proposition \ref{prop:smoothness-relation-modified-L1}, for any set $\commset$, if $p=p^{\meanval,\commset}$, it holds that 
\begin{align*}
     \tilde{\gamma}_g^2(\meanvec;\commset)\ge\frac{\gamma_{g,1}^2(\meanvec;\commset)}{3 + \ln\frac{1}{p_1} + \ln\frac{1}{1-p_\Ndis}}\enspace .
\end{align*}
next, we divide by $\Ndis$ and maximize over $\commset$, which yields
\begin{align*}
     \max_{\commset}\frac{\tilde{\gamma}_g^2(\meanvec;\commset)}{\Ndis}
     \ge
     \max_{\commset}\frac{\gamma_{g,1}^2(\meanvec;\commset)}{\Ndis\br*{3 + \ln\frac{1}{p_1} + \ln\frac{1}{1-p_\Ndis}}}\enspace .
\end{align*}
If the maximizer on the r.h.s. leads to $p_1=0$ or $p_\Ndis=1$, then the r.h.s. of the inequality equals zero for any $\commset\subset\brs*{\Nbatch}$. We can then choose $\commset={i}$ for all $i\in\brs*{\Nbatch}$, and thus $\meanval_i(1-\meanval_i)\nabla_i\rewardvec{\meanvec}^2=0$ for all $i\in\brs*{\Nbatch}$. Therefore, it also holds that $\gamma_{g,2}^2(\meanvec;\emptyset)=0$ and the required inequality trivially holds. Otherwise, $p_1>0$ and $p_\Ndis<1$, and specifically, $p_1>\meanval_{\min}$ and $p_\Ndis<\meanval_{\max}$. Then, it also holds that 
\begin{align*}
     \max_{\commset}\frac{\tilde{\gamma}_g^2(\meanvec;\commset)}{\Ndis}
     \ge
     \max_{\commset}\frac{\gamma_{g,1}^2(\meanvec;\commset)}{\Ndis\br*{3 + \ln\frac{1}{\meanval_{\min}} + \ln\frac{1}{1-\meanval_{\max}}}}\enspace ,
\end{align*}
and applying the first result of the proposition leads to its second result.

\end{proof}

\clearpage


\section{Tightness of the Relations Between the Smoothness Measures}
\label{appendix:tightness}
\subsection{Modified Smoothness and L1 Smoothness}
\label{appendix:tightness modified L1}
In this appendix, we demonstrate the tightness of Proposition \ref{prop:smoothness-relation-modified-L1}. Specifically, we show that there exists a CMAB instance and exponentially small arm parameters $\meanvec$ such that for all $\commset\subset \brs*{\Nbatch}$, it holds that 
$\tilde{\gamma}_g^2(\meanvec;\commset) = \Ocal\br*{\frac{\gamma_{g,1}^2(\meanvec;\commset)}{\Ndis}}$. 
This proves that the logarithmic factor in Inequality \eqref{eq:smoothness-relation-modified-L1} cannot be replaced with a better factor of $\Omega\br*{\frac{1}{\ln \Nbatch}}$.

We start by fixing $\meanval_i=2^{-2(\Nbatch-i)-1}$ and choosing the CMAB instance such that $\nabla_i\rewardvec{\meanvec}=2^{\Nbatch-i}$. Notice that the elements of $\meanvec$ are sorted in an increasing order; therefore, for any $\commset$, the vector $p^{\meanvec,\commset}$ contains all the elements of $\meanvec$ outside $\commset$ in their original order. Thus, we can directly bound $\tilde{\gamma}_g^2(\meanvec;\commset)$ by:
\begin{align*}
     \tilde{\gamma}_g^2(\meanvec;\commset)
     &= \sum_{i\notin\commset}\meanval_i(1-\meanval_i)\nabla_i\rewardvec{\meanvec}^2 + 2\sum_{i\notin\commset}\sum_{j>i,j\notin\commset}\meanval_i(1-\meanval_j)\nabla_i\rewardvec{\meanvec}\nabla_j\rewardvec{\meanvec} \\
     & \stackrel{(1)}{\le} \sum_{i\notin\commset}\meanval_i\nabla_i\rewardvec{\meanvec}^2 + 2\sum_{i\notin\commset}\sum_{j>i,j\notin\commset}\meanval_i\nabla_i\rewardvec{\meanvec}\nabla_j\rewardvec{\meanvec} \\
     & \stackrel{(2)}{\le} 2\sum_{i\notin\commset}\meanval_i\nabla_i\rewardvec{\meanvec}\sum_{j\ge i,j\notin\commset}\nabla_j\rewardvec{\meanvec} \\
     & \stackrel{(3)}{\le} 2\sum_{i\notin\commset}2^{-2(\Nbatch-i)-1}2^{\Nbatch-i}\sum_{j\ge i}2^{\Nbatch-j} \\
     & = \sum_{i\notin\commset}2^{-(\Nbatch-i)}\sum_{j\ge i}2^{\Nbatch-j} \\
     & \le \sum_{i\notin\commset}2^{-(\Nbatch-i)}2^{\Nbatch-i+1}  \\
     &= 2(\Nbatch-\abs*{\commset})=2\Ndis
\end{align*}
In $(1)$ we removed the terms $1-\meanval_i,1-\meanval_j$, which increases the expression since they are smaller then $1$ and the gradients are positive. Similarly, in $(2)$ we multiplied the first term by $2$ and combined the sums into a single term. In $(3)$ we substituted the values of the parameters and increased the internal sum by summing over all $j\ge i$ (including elements in $\commset$). Next, we lower bound $\gamma_{g,1}^2(\meanvec;\commset)$:
\begin{align*}
     \gamma_{g,1}^2(\meanvec;\commset)
     &= \br*{\sum_{i\notin\commset}\sqrt{\meanval_i(1-\meanval_i)}\nabla_i\rewardvec{\meanvec}}^2 
     \stackrel{(1)}{\ge} \frac{1}{2}\br*{\sum_{i\notin\commset}\sqrt{\meanval_i}\nabla_i\rewardvec{\meanvec}}^2 \\
     &= \frac{1}{2}\br*{\sum_{i\notin\commset}2^{-(\Nbatch-i)-\frac{1}{2}}2^{\Nbatch-i}}^2
     =\frac{1}{4}\br*{\sum_{i\notin\commset}1}^2 
     = \frac{\Ndis^2}{4}
\end{align*}
In $(1)$ we used the fact that for all $i$, $1-\meanval_i\ge\frac{1}{2}$. 
Using both bounds, we conclude that for all $\commset\subset\brs*{\Nbatch}$, it holds that
\begin{align*}
     \frac{\tilde{\gamma}_g^2(\meanvec;\commset)}{\gamma_{g,1}^2(\meanvec;\commset)}
     \le \frac{2\Ndis}{\frac{1}{4}\Ndis^2}=\Ocal\br*{\frac{1}{\Ndis}}\enspace .
\end{align*}
Another conclusion from this example is that 
\begin{align*}
     \frac{\max_\commset\tilde{\gamma}_g^2(\meanvec;\commset)}{\max_\commset\gamma_{g,1}^2(\meanvec;\commset)}
     \le \frac{2\Nbatch}{\frac{1}{4}\Nbatch^2}=\Ocal\br*{\frac{1}{\Nbatch}}\enspace .
\end{align*}
Thus, in contrast to the relations between the L1 and L2 Gini-smoothness measure, we cannot improve the inequality by maximizing over the set $\commset$. We end this section by remarking that in this example, one can easily observe that $p_1\le 2^{-2\Ndis+1}$, and thus $\ln\frac{1}{p_1}\approx\Ndis$. Therefore, for this instance, Proposition \ref{prop:smoothness-relation-modified-L1} is tight.

\subsection{L1 Smoothness and L2 Smoothness}
\label{appendix:tightness L1 L2}
In this appendix we prove the tightness of the relation between the L1 and L2 smoothness measures. We do so by proving that Lemma \ref{lemma:normsRelation} is tight up to a constant factor. Let $x\in\R^n$ such that $x_i=\sqrt{i}-\sqrt{i-1}$. Specifically, $x$ is positive and sorted in a decreasing order, and therefore for any $d\in\brs*{n}$ and any $A$ such that $\abs*{A}=d$, the maximal value of  $\frac{\norm{x_A}_1^2}{\abs*{A}}$ is obtained for $A=\brs*{d}$. Moreover, for this specific set, it also holds that
\begin{align*}
     \frac{1}{\abs*{A}}\norm{x_A}_1^2 
     = \frac{1}{d}\br*{\sum_{i=1}^d x_i}^2 
     =  \frac{1}{d}\br*{\sum_{i=1}^d \sqrt{i}-\sqrt{i-1}}^2
     = \frac{1}{d}\br*{\sqrt{d}}^2
     = 1 \enspace.
\end{align*}
Finally, we can write the l.h.s. of the norm inequality by 
\begin{align*}
     \max_A\frac{1}{\abs*{A}}\norm{x_A}_1^2 
     = \max_d\max_{\abs*{A}=d}\frac{1}{\abs*{A}}\norm{x_A}_1^2 
     = \max_d 1
     = 1 \enspace.
\end{align*}
Next, we bound the r.h.s. of the inequality by 
\begin{align*}
     \norm{x_A}_2^2 
     & = \sum_{i=1}^n \br*{\sqrt{i}-\sqrt{i-1}}^2 \\
     & = \sum_{i=1}^n \frac{\br*{\sqrt{i}-\sqrt{i-1}}^2\br*{\sqrt{i}+\sqrt{i-1}}^2}{\br*{\sqrt{i}+\sqrt{i-1}}^2} \\
     & = \sum_{i=1}^n \frac{1}{\br*{\sqrt{i}+\sqrt{i-1}}^2} \\
     & \ge \sum_{i=1}^n \frac{1}{4i} \\
     & \ge \frac{\ln(n+1)}{4}
\end{align*}
Thus, for this example, $\max_A\frac{1}{\abs*{A}}\norm{x_A}_1^2 \le \frac{4}{\ln(n+1)}\norm{x_A}_2^2$, and Lemma \ref{lemma:normsRelation} is tight up to a constant factor.

\end{document}